\documentclass{article} 
\usepackage{iclr2027_conference,times}

\usepackage{natbib}
\usepackage{times}
\usepackage{latexsym}
\usepackage[T1]{fontenc}    
\usepackage[utf8]{inputenc} 
\usepackage{graphicx}
\usepackage{inconsolata}    
\usepackage{hyperref}       
\hypersetup{hidelinks, linktoc=all, colorlinks=false} 
\usepackage{xr-hyper}       
\usepackage{xurl}           
\usepackage{booktabs}       
\usepackage{longtable}      
\usepackage{amsfonts}       
\usepackage{amsthm}         
\usepackage{nicefrac}       
\usepackage{microtype}      
\usepackage[table]{xcolor}  
\usepackage{multirow}
\usepackage[most]{tcolorbox}  
\tcbuselibrary{listings}      
\tcbuselibrary{skins}         
\usepackage{footnote}
\usepackage{xspace}         
\usepackage{mfirstuc}       
\usepackage{tabulary}       
\usepackage{wrapfig}        
\usepackage{mathpartir}     
\usepackage{amsmath}        
\usepackage{amssymb}        
\usepackage{bbm}            
\usepackage[font=footnotesize,labelfont=bf]{caption} 
\usepackage[font=footnotesize]{subcaption}     
\usepackage{listings}       
\usepackage{outlines}       
\usepackage{tikz}           
\usetikzlibrary{calc,positioning,fit,tikzmark,shapes.multipart,shapes,arrows.meta,backgrounds, decorations.text, decorations.markings} 
\usepackage{pifont}
\usepackage{titletoc}       
\usepackage[only,llbracket,rrbracket]{stmaryrd}
\usepackage{mathrsfs}
\usepackage{mathtools}      
\usepackage{pgf-pie}        
\usepackage{wheelchart}     
\usepackage{pgfplots}
\pgfplotsset{compat=1.18}
\usepackage{enumitem}       
\usepackage{placeins}       
\usepackage{threeparttable} 
\usepackage{xparse}         
\usepackage{bm}             
\usepackage{siunitx}        
\usepackage{colortbl}       
\usepackage{mathtools}      
\usepackage{adjustbox}      
\usepackage{xfp}      
\usepackage[symbol]{footmisc} 
\usepackage{makecell}       
\usepackage{algorithm}      
\usepackage{algpseudocode}  
\usepackage[noEnd=true,indLines=true,spaceRequire=false]{algpseudocodex} 
\alglanguage{pseudocode}    
\usepackage{bigdelim}       
\usepackage{contour}        
\usepackage{dsfont}         
\usepackage[noblocks]{authblk}        

\definecolor{eqnBoxFill}{RGB}{253,250,244}
\newtcolorbox{eqnBox}{
  shadow=false,
  boxrule=0.2pt,
  colframe=black!60,
  colback=eqnBoxFill,
  arc=2pt,
  left=1pt,
  right=1pt,
  top=0pt,
  bottom=4pt,
  halign=center,
  valign=center,
  before skip=4pt,
  after skip=4pt,
}
\newtheoremstyle{paper-theorem}
    {4pt}                         
    {4pt}                         
    {\itshape}                    
    {}                            
    {\bfseries}                   
    {}                            
    {0.5em}                       
    {\thmname{#1}\thmnumber{ #2.}\thmnote{ \normalfont[#3]}}
\newtheoremstyle{paper-definition}
    {4pt}                         
    {4pt}                         
    {\itshape}                 
    {}                            
    {\bfseries}                   
    {}                            
    {0.5em}                       
    {\thmname{#1}\thmnumber{ #2.}\thmnote{ \normalfont[#3]}}

\theoremstyle{paper-theorem}
\newtheorem{theorem}{Theorem}[section]
\newtheorem{proposition}[theorem]{Proposition}
\newtheorem{lemma}[theorem]{Lemma}
\newtheorem{corollary}[theorem]{Corollary}

\theoremstyle{paper-definition}
\newtheorem{definition}[theorem]{Definition}

\definecolor{theoremBoxFill}{RGB}{247,247,247}
\tcbset{
    theorem box/.style={
        enhanced,
        breakable,
        colback=theoremBoxFill,
        colframe=theoremBoxFill,
        boxrule=0pt,
        arc=2pt,
        left=3pt,
        right=3pt,
        top=4pt,
        bottom=4pt,
        before skip=6pt,
        after skip=6pt
    }
}

\tcolorboxenvironment{theorem}{theorem box}
\tcolorboxenvironment{proposition}{theorem box}
\tcolorboxenvironment{lemma}{theorem box}
\tcolorboxenvironment{corollary}{theorem box}
\tcolorboxenvironment{definition}{theorem box}
\tcolorboxenvironment{notation}{theorem box}
\tcolorboxenvironment{remark}{theorem box}

\newcommand{\DefMacro}[2]{\expandafter\newcommand\csname rmk-#1\endcsname{#2}}
\newcommand{\UseMacro}[1]{\csname rmk-#1\endcsname}

\newcommand{\tool}{\textsc{Twister}\xspace}%
\newcommand{\Title}{Twist, Don't Tilt: Trajectory-Exact Constra-\newline ined Decoding for Masked Diffusion Models}

\definecolor{lightred}{RGB}{165,28,48}
\definecolor{lightpurple}{RGB}{101,1,107}
\definecolor{lightblue}{RGB}{0,84,147}
\definecolor{lightgreen}{RGB}{0,120,90}

\newcommand{\dfa}{DFA\xspace}

\newcommand{\llm}{LLM\xspace}
\newcommand{\llms}{LLMs\xspace}
\newcommand{\mdlm}{MDLM\xspace}
\newcommand{\mdlms}{MDLMs\xspace}

\newcommand{\crf}{CRF\xspace}

\newcommand{\smc}{SMC\xspace}

\newcommand{\ffbs}{FFBS\xspace}
\newcommand{\ffbsKernel}{step-exact\xspace}
\newcommand{\ffbsKernelProp}{Step-Exact\xspace}
\newcommand{\fkc}{Feynman--Kac\xspace}
\newcommand{\doob}{Doob\xspace}

\newcommand{\Code}[1]{{\ifmmode{\mathtt{#1}}\else$\mathtt{#1}$\fi}}

\newcommand{\eg}{e.g.,\xspace}
\newcommand{\ie}{i.e.,\xspace}
\newcommand{\etc}{etc.\xspace}
\newcommand{\MyPara}[1]{\vspace{0.2mm}\noindent\textbf{#1}.\hspace{1pt}}
\newcommand{\TableFontSize}{\footnotesize}
\newcommand{\AlgorithmFontSize}{\footnotesize}

\DeclareRobustCommand{\circledCaption}[1]{\tikz[baseline=(char.base)]{\node[shape=circle, draw, minimum size=4pt, inner sep=1pt, fill=black, text=white] (char) {\footnotesize #1};}}
\DeclareRobustCommand{\circledCaptionTiny}[1]{\tikz[baseline=(char.base)]{\node[shape=circle, draw, minimum size=1.5pt, inner sep=0.8pt, fill=black, text=white] (char) {\scriptsize #1};}}

\newcommand{\dingo}{\text{DINGO}\xspace}

\newcommand{\outlines}{Outlines\xspace}

\newcommand{\fAlpha}{\Sigma}
\newcommand{\fAlphaStar}{\Sigma^{*}}

\newcommand{\fWordW}[1]{\bm{w_{#1}}}
\newcommand{\fMask}{\bot}

\newcommand{\fLangSet}{\mathcal{C}}

\newcommand{\fBigO}[1]{\mathcal{O}(#1)}

\newcommand{\fDiffusionStep}{t}
\newcommand{\fDiffusionStepMax}{\mathit{T}}
\newcommand{\fPosition}{i}
\newcommand{\fPositionMax}{\mathit{N}}
\newcommand{\fMaskFunc}[1]{\mathcal{M}(#1)}
\newcommand{\fMaskDiffFunc}[2]{\mathcal{R}(#1,#2)}
\newcommand{\fVocab}{\mathcal{V}}

\newcommand{\fDiffStateSpace}{\mathcal{X}}
\newcommand{\fDiffStateRV}[1]{\textbf{X}_{#1}}
\newcommand{\fDiffState}[1]{\textbf{x}_{#1}}
\newcommand{\fDiffStateFrozen}{\textbf{x}^{\ast}}
\newcommand{\fDiffCleanStateRV}[1]{\textbf{Y}_{#1}}
\newcommand{\fDiffCleanState}[1]{\textbf{y}_{#1}}
\newcommand{\fAutomtnStateChainRV}{\textbf{Q}}
\newcommand{\fAutomtnStateChain}{\textbf{q}}
\newcommand{\fAutomtnStateChainToken}[1]{\fAutomtnStateChain(#1)}
\newcommand{\fDiffStateToken}[3]{\fDiffState{#1}{#3}(#2)}
\newcommand{\fDiffCleanStateToken}[3]{\fDiffCleanState{#1}{#3}(#2)}
\newcommand{\fDiffCleanToken}[1]{v_{#1}}
\newcommand{\fDiffKernelSymbol}{\kappa}
\newcommand{\fDiffKernel}[3]{\fDiffKernelSymbol_{#1}(#2\mid#3)}
\newcommand{\fDiffFrozenKernel}[4]{\fDiffKernelSymbol^{(#4)}_{#1}(#2\mid#3)}

\newcommand{\fProb}{\mathit{P}}

\newcommand{\fDiffCategorical}[3]{\mathrm{Cat}_{#1}(#2\mid#3)}
\newcommand{\fDenoiserPosterior}[2]{\mu_{}(#1\mid#2)}

\newcommand{\fSimplex}[1]{\Delta^{#1}}
\newcommand{\fRemaskingSetRV}[1]{\textbf{R}_{#1}}
\newcommand{\fRemaskingSet}[1]{r_{#1}}
\newcommand{\fDiffScheduler}[3]{s_{#1}(#2\mid#3)}
\newcommand{\fTargetProb}{\pi}
\newcommand{\fTargetProbUNorm}{\widetilde{\fTargetProb}}

\newcommand{\fTargetProbUNormFunc}[2]{\fTargetProbUNorm_{#1}(#2)}

\newcommand{\fPartitionFuncClamp}[3]{\mathit{Z}_{#1}(#2;#3)}

\newcommand{\fFeynmanKacKernelSymbol}{\mathit{M}}
\newcommand{\fFeynmanKacPotentialSymbol}{\mathit{G}}
\newcommand{\fFeynmanKacKernel}[3]{\fFeynmanKacKernelSymbol_{#1}(#2\mid#3)}
\newcommand{\fFeynmanKacPotential}[3]{\fFeynmanKacPotentialSymbol_{#1}(#2,#3)}
\newcommand{\fFeynmanKacTwistFunction}[1]{\eta_{#1}}
\newcommand{\fSMCParticleIndex}{k}
\newcommand{\fSMCParticleIndexMax}{\mathit{K}}
\newcommand{\fSMCParticleWeight}[1]{\mathit{W}_{#1}}
\newcommand{\fSMCParticleWeightUNorm}[1]{\widetilde{\mathit{W}}_{#1}}
\newcommand{\fDoobHarmonic}[1]{h_{#1}}
\newcommand{\fFrozenDoobHarmonic}[3]{\widehat{h}_{#1}^{(#2)}(#3)}
\newcommand{\fDoobKernelSymbol}{\fDiffKernelSymbol^{\scalebox{0.6}{$\bigstar$}}}
\newcommand{\fDoobKernel}[3]{\fDoobKernelSymbol_{#1}(#2\mid#3)}

\newcommand{\fDoobNormaliser}{Z^{\scalebox{0.6}{$\bigstar$}}}

\newcommand{\fFFBSKernelSymbol}{\fDiffKernelSymbol^{\scalebox{0.56}{\,\bm{$\fAutomtn$}}}}
\newcommand{\fFFBSKernel}[3]{\fFFBSKernelSymbol_{#1}(#2\mid#3)}

\newcommand{\fAutomtn}{\mathcal{A}}
\newcommand{\fAutomtnTuple}{\fTuple{\fAlpha, \fStates, \fState{0}, \fStatesAcc, \fTrans}}              
\newcommand{\fStates}{\mathcal{Q}}       

\newcommand{\fState}[1]{q_{\scalebox{0.7}{$#1$}}} 
        
\newcommand{\fStatesAcc}{\mathcal{F}}                             
\newcommand{\fTrans}{\delta}                      
\newcommand{\fTransClosure}{\fTrans^{*}}
\newcommand{\fTransFunc}[2]{\fTrans(#1, #2)}
\newcommand{\fTransClosureFunc}[2]{\fTransClosure(#1, #2)}
\newcommand{\fAutomtnConstraint}[3]{\psi_{#1}(#2,#3)}
\newcommand{\fTuple}[1]{\langle #1 \rangle}
\newcommand{\fSize}[1]{\lvert #1 \rvert}
\newcommand{\fSet}[1]{\{#1\}}

\newcommand{\fIndicatorFunc}[1]{\llbracket #1 \rrbracket}

\newcommand{\fPositionEquiv}[3]{\llbracket #1 \equiv_{#3} #2 \rrbracket}
\newcommand{\fNativePathLaw}[3]{p_{#1}^{\scalebox{0.5}{\textsf{mdlm}}}(#2\mid#3)}
\newcommand{\fDoobPathLaw}[3]{p_{#1}^{\scalebox{0.6}{$\bigstar$}}(#2\mid#3;\fLangSet)}
\newcommand{\fFFBSPathLaw}[3]{p_{#1}^{\scalebox{0.56}{\,\bm{$\fAutomtn$}}}(#2\mid#3;\fLangSet)}
\newcommand{\fProbMeasure}[2]{\mathbb{P}_{#2}(#1)}

\newcommand{\llada}{\textsc{LLaDA-8B-Base}\xspace}
\newcommand{\dream}{\textsc{Dream-7B-Base}\xspace}
\newcommand{\dreamCoder}{\textsc{DreamCoder}\xspace}
\newcommand{\lladaInstruct}{\textsc{LLaDA-8B-Inst}\xspace}
\newcommand{\dreamInstruct}{\textsc{Dream-7B-Inst}\xspace}
\newcommand{\dreamCoderInstruct}{\textsc{DreamCoder-7B-Inst}\xspace}
\newcommand{\lladaInst}{\textsc{LLaDA-8B-Inst}\xspace}
\newcommand{\dreamInst}{\textsc{Dream-7B-Inst}\xspace}
\newcommand{\dreamCoderInst}{\textsc{DreamCoder-7B-Inst}\xspace}

\newcommand{\kuleshov}{\textsc{OWT-130M}\xspace}
\newcommand{\lladaText}{LLaDA\xspace}
\newcommand{\dreamText}{\textsc{DREAM}\xspace}
\newcommand{\dreamCoderText}{\textsc{DreamCoder-7B-Base}\xspace}

\DefMacro{table-json-mode-table}{table*}
\DefMacro{table-json-mode-Model}{\textbf{Model}}
\DefMacro{table-json-mode-Method}{\textbf{Method}}
\DefMacro{table-json-mode-Parse}{\textbf{Parse Valid (\%)}}
\DefMacro{table-json-mode-Schema}{\textbf{Schema Valid (\%)}}
\DefMacro{table-json-mode-Time}{\textbf{Time (s)}}
\DefMacro{table-json-mode-Dream-v0-Base-7B-dingo-n}{100}
\DefMacro{table-json-mode-Dream-v0-Base-7B-dingo-parse}{100}
\DefMacro{table-json-mode-Dream-v0-Base-7B-dingo-schema}{98}
\DefMacro{table-json-mode-Dream-v0-Base-7B-dingo-time-mean-Dream-v0-Base-7B-dingo-time-std}{$22 \pm 43$}
\DefMacro{table-json-mode-Dream-v0-Base-7B-cedar-marginal-n}{100}
\DefMacro{table-json-mode-Dream-v0-Base-7B-cedar-marginal-parse}{100}
\DefMacro{table-json-mode-Dream-v0-Base-7B-cedar-marginal-schema}{98}
\DefMacro{table-json-mode-Dream-v0-Base-7B-cedar-marginal-time-mean-Dream-v0-Base-7B-cedar-marginal-time-std}{$22 \pm 43$}
\DefMacro{table-json-mode-Dream-v0-Base-7B-cedar-smc1-n}{100}
\DefMacro{table-json-mode-Dream-v0-Base-7B-cedar-smc1-parse}{100}
\DefMacro{table-json-mode-Dream-v0-Base-7B-cedar-smc1-schema}{98}
\DefMacro{table-json-mode-Dream-v0-Base-7B-cedar-smc1-time-mean-Dream-v0-Base-7B-cedar-smc1-time-std}{$32 \pm 65$}
\DefMacro{table-json-mode-Dream-v0-Base-7B-cedar-smc4-n}{100}
\DefMacro{table-json-mode-Dream-v0-Base-7B-cedar-smc4-parse}{100}
\DefMacro{table-json-mode-Dream-v0-Base-7B-cedar-smc4-schema}{98}
\DefMacro{table-json-mode-Dream-v0-Base-7B-cedar-smc4-time-mean-Dream-v0-Base-7B-cedar-smc4-time-std}{$42 \pm 86$}
\DefMacro{table-json-mode-Dream-v0-Instruct-7B-dingo-n}{100}
\DefMacro{table-json-mode-Dream-v0-Instruct-7B-dingo-parse}{100}
\DefMacro{table-json-mode-Dream-v0-Instruct-7B-dingo-schema}{99}
\DefMacro{table-json-mode-Dream-v0-Instruct-7B-dingo-time-mean-Dream-v0-Instruct-7B-dingo-time-std}{$21 \pm 45$}
\DefMacro{table-json-mode-Dream-v0-Instruct-7B-cedar-marginal-n}{100}
\DefMacro{table-json-mode-Dream-v0-Instruct-7B-cedar-marginal-parse}{100}
\DefMacro{table-json-mode-Dream-v0-Instruct-7B-cedar-marginal-schema}{99}
\DefMacro{table-json-mode-Dream-v0-Instruct-7B-cedar-marginal-time-mean-Dream-v0-Instruct-7B-cedar-marginal-time-std}{$21 \pm 46$}
\DefMacro{table-json-mode-Dream-v0-Instruct-7B-cedar-smc1-n}{100}
\DefMacro{table-json-mode-Dream-v0-Instruct-7B-cedar-smc1-parse}{100}
\DefMacro{table-json-mode-Dream-v0-Instruct-7B-cedar-smc1-schema}{99}
\DefMacro{table-json-mode-Dream-v0-Instruct-7B-cedar-smc1-time-mean-Dream-v0-Instruct-7B-cedar-smc1-time-std}{$33 \pm 72$}
\DefMacro{table-json-mode-Dream-v0-Instruct-7B-cedar-smc4-n}{100}
\DefMacro{table-json-mode-Dream-v0-Instruct-7B-cedar-smc4-parse}{100}
\DefMacro{table-json-mode-Dream-v0-Instruct-7B-cedar-smc4-schema}{99}
\DefMacro{table-json-mode-Dream-v0-Instruct-7B-cedar-smc4-time-mean-Dream-v0-Instruct-7B-cedar-smc4-time-std}{$42 \pm 87$}
\DefMacro{table-json-mode-Dream-Coder-v0-Base-7B-dingo-n}{100}
\DefMacro{table-json-mode-Dream-Coder-v0-Base-7B-dingo-parse}{100}
\DefMacro{table-json-mode-Dream-Coder-v0-Base-7B-dingo-schema}{98}
\DefMacro{table-json-mode-Dream-Coder-v0-Base-7B-dingo-time-mean-Dream-Coder-v0-Base-7B-dingo-time-std}{$21 \pm 38$}
\DefMacro{table-json-mode-Dream-Coder-v0-Base-7B-cedar-marginal-n}{100}
\DefMacro{table-json-mode-Dream-Coder-v0-Base-7B-cedar-marginal-parse}{100}
\DefMacro{table-json-mode-Dream-Coder-v0-Base-7B-cedar-marginal-schema}{98}
\DefMacro{table-json-mode-Dream-Coder-v0-Base-7B-cedar-marginal-time-mean-Dream-Coder-v0-Base-7B-cedar-marginal-time-std}{$21 \pm 38$}
\DefMacro{table-json-mode-Dream-Coder-v0-Base-7B-cedar-smc1-n}{100}
\DefMacro{table-json-mode-Dream-Coder-v0-Base-7B-cedar-smc1-parse}{100}
\DefMacro{table-json-mode-Dream-Coder-v0-Base-7B-cedar-smc1-schema}{98}
\DefMacro{table-json-mode-Dream-Coder-v0-Base-7B-cedar-smc1-time-mean-Dream-Coder-v0-Base-7B-cedar-smc1-time-std}{$31 \pm 57$}
\DefMacro{table-json-mode-Dream-Coder-v0-Base-7B-cedar-smc4-n}{100}
\DefMacro{table-json-mode-Dream-Coder-v0-Base-7B-cedar-smc4-parse}{100}
\DefMacro{table-json-mode-Dream-Coder-v0-Base-7B-cedar-smc4-schema}{98}
\DefMacro{table-json-mode-Dream-Coder-v0-Base-7B-cedar-smc4-time-mean-Dream-Coder-v0-Base-7B-cedar-smc4-time-std}{$40 \pm 73$}
\DefMacro{table-json-mode-Dream-Coder-v0-Instruct-7B-dingo-n}{100}
\DefMacro{table-json-mode-Dream-Coder-v0-Instruct-7B-dingo-parse}{100}
\DefMacro{table-json-mode-Dream-Coder-v0-Instruct-7B-dingo-schema}{99}
\DefMacro{table-json-mode-Dream-Coder-v0-Instruct-7B-dingo-time-mean-Dream-Coder-v0-Instruct-7B-dingo-time-std}{$22 \pm 47$}
\DefMacro{table-json-mode-Dream-Coder-v0-Instruct-7B-cedar-marginal-n}{100}
\DefMacro{table-json-mode-Dream-Coder-v0-Instruct-7B-cedar-marginal-parse}{100}
\DefMacro{table-json-mode-Dream-Coder-v0-Instruct-7B-cedar-marginal-schema}{99}
\DefMacro{table-json-mode-Dream-Coder-v0-Instruct-7B-cedar-marginal-time-mean-Dream-Coder-v0-Instruct-7B-cedar-marginal-time-std}{$22 \pm 47$}
\DefMacro{table-json-mode-Dream-Coder-v0-Instruct-7B-cedar-smc1-n}{100}
\DefMacro{table-json-mode-Dream-Coder-v0-Instruct-7B-cedar-smc1-parse}{100}
\DefMacro{table-json-mode-Dream-Coder-v0-Instruct-7B-cedar-smc1-schema}{99}
\DefMacro{table-json-mode-Dream-Coder-v0-Instruct-7B-cedar-smc1-time-mean-Dream-Coder-v0-Instruct-7B-cedar-smc1-time-std}{$34 \pm 76$}
\DefMacro{table-json-mode-Dream-Coder-v0-Instruct-7B-cedar-smc4-n}{100}
\DefMacro{table-json-mode-Dream-Coder-v0-Instruct-7B-cedar-smc4-parse}{100}
\DefMacro{table-json-mode-Dream-Coder-v0-Instruct-7B-cedar-smc4-schema}{98}
\DefMacro{table-json-mode-Dream-Coder-v0-Instruct-7B-cedar-smc4-time-mean-Dream-Coder-v0-Instruct-7B-cedar-smc4-time-std}{$42 \pm 88$}
\DefMacro{table-json-mode-LLaDA-Base-8B-dingo-n}{100}
\DefMacro{table-json-mode-LLaDA-Base-8B-dingo-parse}{100}
\DefMacro{table-json-mode-LLaDA-Base-8B-dingo-schema}{98}
\DefMacro{table-json-mode-LLaDA-Base-8B-dingo-time-mean-LLaDA-Base-8B-dingo-time-std}{$19 \pm 33$}
\DefMacro{table-json-mode-LLaDA-Base-8B-cedar-marginal-n}{100}
\DefMacro{table-json-mode-LLaDA-Base-8B-cedar-marginal-parse}{100}
\DefMacro{table-json-mode-LLaDA-Base-8B-cedar-marginal-schema}{98}
\DefMacro{table-json-mode-LLaDA-Base-8B-cedar-marginal-time-mean-LLaDA-Base-8B-cedar-marginal-time-std}{$19 \pm 33$}
\DefMacro{table-json-mode-LLaDA-Base-8B-cedar-smc1-n}{100}
\DefMacro{table-json-mode-LLaDA-Base-8B-cedar-smc1-parse}{100}
\DefMacro{table-json-mode-LLaDA-Base-8B-cedar-smc1-schema}{98}
\DefMacro{table-json-mode-LLaDA-Base-8B-cedar-smc1-time-mean-LLaDA-Base-8B-cedar-smc1-time-std}{$27 \pm 50$}
\DefMacro{table-json-mode-LLaDA-Base-8B-cedar-smc4-n}{100}
\DefMacro{table-json-mode-LLaDA-Base-8B-cedar-smc4-parse}{100}
\DefMacro{table-json-mode-LLaDA-Base-8B-cedar-smc4-schema}{98}
\DefMacro{table-json-mode-LLaDA-Base-8B-cedar-smc4-time-mean-LLaDA-Base-8B-cedar-smc4-time-std}{$36 \pm 61$}
\DefMacro{table-json-mode-LLaDA-Instruct-8B-dingo-n}{100}
\DefMacro{table-json-mode-LLaDA-Instruct-8B-dingo-parse}{100}
\DefMacro{table-json-mode-LLaDA-Instruct-8B-dingo-schema}{99}
\DefMacro{table-json-mode-LLaDA-Instruct-8B-dingo-time-mean-LLaDA-Instruct-8B-dingo-time-std}{$18 \pm 37$}
\DefMacro{table-json-mode-LLaDA-Instruct-8B-cedar-marginal-n}{100}
\DefMacro{table-json-mode-LLaDA-Instruct-8B-cedar-marginal-parse}{100}
\DefMacro{table-json-mode-LLaDA-Instruct-8B-cedar-marginal-schema}{99}
\DefMacro{table-json-mode-LLaDA-Instruct-8B-cedar-marginal-time-mean-LLaDA-Instruct-8B-cedar-marginal-time-std}{$21 \pm 42$}
\DefMacro{table-json-mode-LLaDA-Instruct-8B-cedar-smc1-n}{100}
\DefMacro{table-json-mode-LLaDA-Instruct-8B-cedar-smc1-parse}{100}
\DefMacro{table-json-mode-LLaDA-Instruct-8B-cedar-smc1-schema}{99}
\DefMacro{table-json-mode-LLaDA-Instruct-8B-cedar-smc1-time-mean-LLaDA-Instruct-8B-cedar-smc1-time-std}{$29 \pm 62$}
\DefMacro{table-json-mode-LLaDA-Instruct-8B-cedar-smc4-n}{100}
\DefMacro{table-json-mode-LLaDA-Instruct-8B-cedar-smc4-parse}{100}
\DefMacro{table-json-mode-LLaDA-Instruct-8B-cedar-smc4-schema}{99}
\DefMacro{table-json-mode-LLaDA-Instruct-8B-cedar-smc4-time-mean-LLaDA-Instruct-8B-cedar-smc4-time-std}{$37 \pm 73$}

\newcommand{\TableJsonModeDatasetCaption}{%
Constraint satisfaction on \jsonModeEval for all decoding methods. The \textit{Parse Valid (\%)} and \textit{Schema Valid (\%)} columns measure whether outputs are valid JSON and satisfy the schema set by the \datapoint; \textit{Time (s)}
measures the average time taken to generate the output.
\label{tab:json-mode-dataset}
}

\newcommand{\AtsmcStage}[3]{%
\tikz[baseline=(stage.base)]{
    \node[
        anchor=base west,
        rounded corners=1.2pt,
        inner xsep=3.2pt,
        inner ysep=1.2pt,
        line width=0.4pt,
        text=#1!70,
        font=\bfseries\footnotesize
    ] (stage) {\strut #2};
}}
\alglanguage{pseudocodex}
\makeatletter
\algnewcommand\algorithmicparallelfor{\textbf{parallel for}}
\algdef{SE}[PARFOR]{ParallelFor}{EndParallelFor}[1]{%
    \algpx@startCodeCommand\algpx@startIndent
    \algorithmicparallelfor\ #1\ \algorithmicdo%
}{%
    \algpx@startEndBlockCommand\algpx@endIndent
    \algorithmicend\ \algorithmicparallelfor%
}
\algtext*{EndParallelFor}
\pretocmd{\ParallelFor}{\algpx@endCodeCommand}{}{}
\pretocmd{\EndParallelFor}{\algpx@endCodeCommand[1]\algpx@endIndent}{}{}
\makeatother

\newcommand{\FigIntroKuleshovBiasCaption}{%
    \kuleshov
}

\newcommand{\FigIntroBiasCaption}{%
    Trajectory bias in constrained decoding of \mdlms captured by the Total Variation Distance (TVD) between the \doob $\fDoobHarmonic{}$-transformed native path law (rejection sampling) and the step-exact automaton-constrained distribution (\cite{dang2026constrained}, $\gnardSmcFour$ and $\gnardSmcEight$) across denoising steps. For "a.*".
}
\newcommand{\FigTrajectoryBiasCaption}{%
A three-token example of trajectory bias.
Let $\fVocab\!=\!\fSet{a,b}$, $\fLangSet\!=\!a^{*}b^{+}$,
$\fDiffusionStepMax\!=\!2$, and
$\fDiffState{0}\!=\!(\fMask,\fMask,b)$, with position $2$ revealed before
position $1$.
Revealing $a$ gives $\fDiffState{1}^{a}\!=\!(\fMask,a,b)$, with $aab$ as its only valid completion under $\fLangSet$, while revealing $b$ gives $\fDiffState{1}^{b}\!=\!(\fMask,b,b)$, whose completions are both valid.
Assuming uniform categoricals at $\fDiffState{0}$ (\circledCaptionTiny{1}),
$\fPartitionFuncClamp{}{\fDiffState{1}^{a}}{\fDiffState{0}}\!=\!0.5$
(\circledCaptionTiny{2}) and
$\fPartitionFuncClamp{}{\fDiffState{1}^{b}}{\fDiffState{0}}\!=\!1$,
so the step-exact decoder chooses $\fDiffState{1}^{a}$ with probability $0.5/(0.5\!+\!1)\!=\!1/3$.
After committing $\fDiffState{1}^{a}$, the denoiser rescores
position $1$ as $\fDiffCategorical{1}{a}{\fDiffState{1}^{a}}\!=\!0.9$
(\circledCaptionTiny{3}), so $\fDiffState{1}^{a}$'s actual valid mass
is $\fPartitionFuncClamp{}{\fDiffState{1}^{a}}{\fDiffState{1}^{a}}\!=\!0.9$; which is what the Doob path law uses to choose $\fDiffState{1}^{a}$, $\fDoobHarmonic{1}(\fDiffState{1}^{a})\!=\!0.9$, \ie with probability
$0.9/(0.9\!+\!1)\!=\!9/19$. But step-exact decoding only renormalizes by $\fPartitionFuncClamp{}{\fDiffState{1}^{a}}{\fDiffState{1}^{a}}$
in the next step, thus never correcting $\fPartitionFuncClamp{}{\fDiffState{1}^{a}}{\fDiffState{0}}$.
Since $aab$ is the only valid completion of $\fDiffState{1}^{a}$,
$\fFFBSPathLaw{1:2}{(\fDiffState{1}^{a},aab)}{\fDiffState{0}}\!=\!1/3$ but
$\fDoobPathLaw{1:2}{(\fDiffState{1}^{a},aab)}{\fDiffState{0}}\!=\!9/19$,
due to the per-step mismatch
$\fPartitionFuncClamp{}{\fDiffState{1}^{a}}{\fDiffState{0}}/
\fPartitionFuncClamp{}{\fDiffState{1}^{a}}{\fDiffState{1}^{a}}\!=\!5/9$
in Eq.~\eqref{eq:clamped-mass-tilt}.
}

\newcommand{\AlgorithmTwisterCaption}{%
\tool for step-exact decoding (adaptive resampling shown in Appendix~\ref{app:twister-full}).
\label{alg:twister}
}
\newcommand{\AlgorithmTwisterFullCaption}{%
\tool for step-exact decoding (with adaptive resampling).
\label{alg:twister-full}
}

\newcommand{\gnard}{\tool}

\newcommand{\gnardSmcOne}{\gnard_{\texttt{smc1}}\xspace}
\newcommand{\gnardSmcFour}{\gnard_{\texttt{smc4}}\xspace}
\newcommand{\gnardSmcEight}{\gnard_{\texttt{smc8}}\xspace}
\newcommand{\OurDingo}{DINGO\ensuremath{\dagger}\xspace}
\newcommand{\OurDangErmon}{\citeauthor{dang2026constrained}\ensuremath{\dagger}\xspace}

\newcommand{\jsonModeEval}{JSON-Mode-Eval\xspace}

\newcommand{\regex}{regular expression\xspace}
\newcommand{\datapoint}{datapoint\xspace}
\newcommand{\datapoints}{datapoints\xspace}

\newcommand{\dfaStatesMin}{34\xspace}
\newcommand{\dfaStatesMean}{188\xspace}
\newcommand{\dfaStatesMax}{625\xspace}

\title{\Title}
\author{%
  \textbf{Aditya Thimmaiah}$^{1}$,\, 
  \textbf{Lara Marinov}$^{1}$,\, 
  \textbf{Jayanth Srinivasa}$^{2}$,\, 
  \authorcr
  \textbf{Haris Vikalo}$^{1}$,\, 
  \textbf{Junyi Jessy Li}$^{1}$,\, 
  \textbf{Milos Gligoric}$^{1}$
  \authorcr\medskip
  \normalfont\texttt{auditt@utexas.edu, marinov@utexas.edu, jasriniv@cisco.com,}\authorcr
  \normalfont\texttt{hvikalo@ece.utexas.edu, jessy@utexas.edu, gligoric@utexas.edu}
  \authorcr\medskip
  \normalfont
  $^{1}$The University of Texas at Austin
  \quad
  $^{2}$Cisco Research
}

\iclrfinalcopy
\date{}

\begin{document}
\maketitle

\begin{abstract}
Constrained decoding for Masked Diffusion Language Models (\mdlms)
aims to ensure that generated outputs satisfy a specified structure or
syntax constraint.
\mdlms generate outputs by repeatedly unmasking masked positions
present in their current state.
Recent strategies for constrained decoding constrain the model's
per-step mean-field posterior (which factorizes over masked positions) by
enforcing the desired constraint with an automaton. The resulting
chain-structured factor graph allows exact constrained sampling via
dynamic programming.
%
However, despite each draw being exact and
constraint-satisfying, we prove that their composition, in general, tilts away from
the model's relative probabilities over valid trajectories, thus leading
to trajectory bias.
We derive an exact expression for this bias as a product of ratios
measuring how valid continuation mass changes when the denoiser is
reconditioned, and characterize when the bias vanishes.
We then correct the bias by introducing \tool, the first
automaton-twisted Sequential Monte Carlo decoder for \mdlms, using the
step-exact decoder as the proposal.
We show that for regular language constraints, the \fkc
correction is exactly computable, with the twists obtained efficiently using
quantities pre-computed for step-exact sampling.
We prove that the resulting \fkc model targets the unbiased
\doob $\fDoobHarmonic{}$-transformed path law conditioned on constraint satisfaction.
\end{abstract}

\begin{figure}[h]
    \centering
    \vspace{-2mm}
    \includegraphics[width=0.95\textwidth]{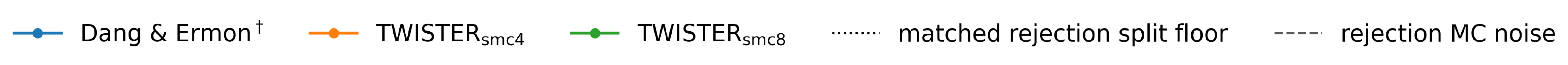}\\
    \vspace{0.5em}
    \begin{subfigure}[b]{0.308\textwidth}
        \includegraphics[width=\textwidth]{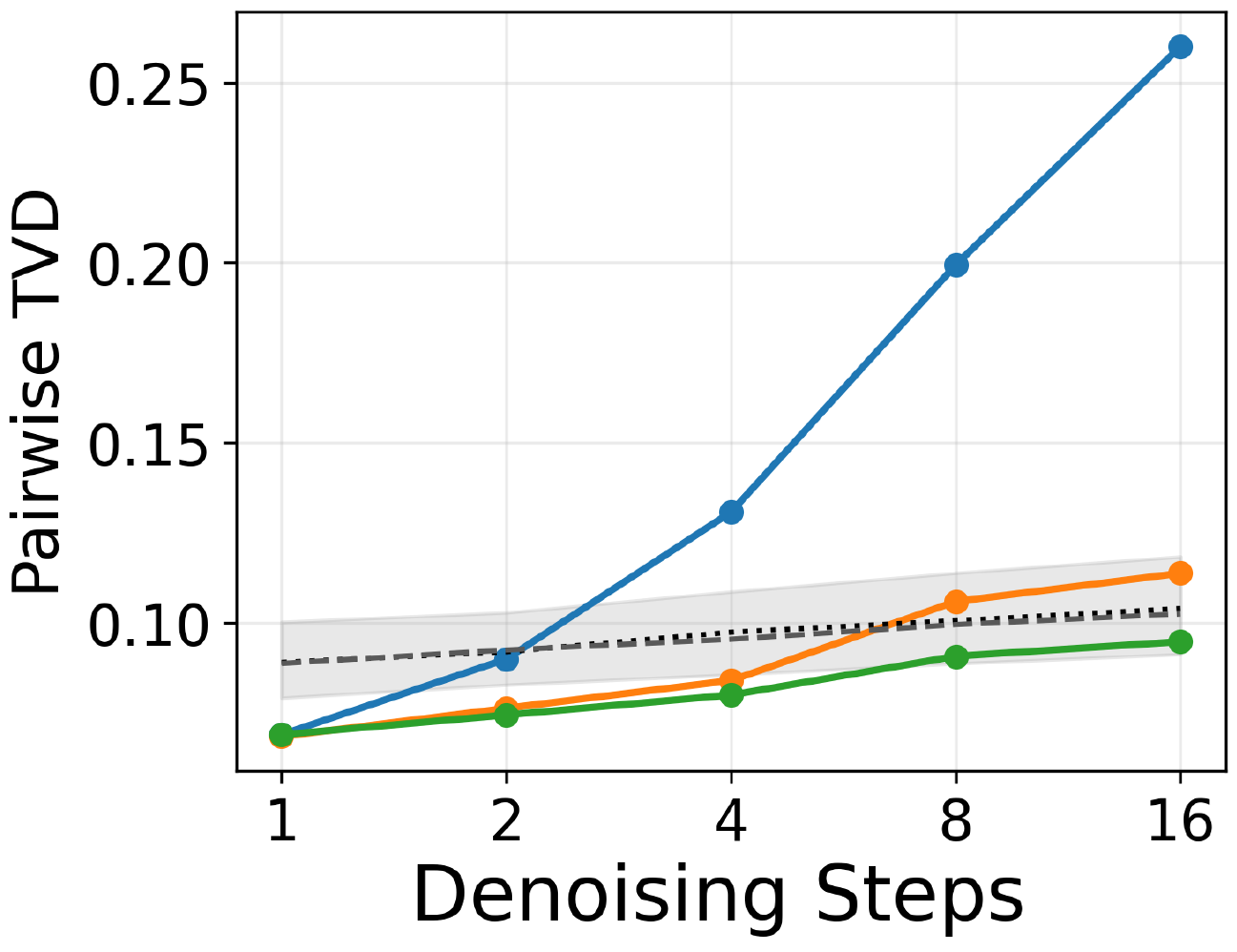}
        \caption{\FigIntroKuleshovBiasCaption}
        \label{fig:intro:a}
    \end{subfigure}
    \hfill
    \begin{subfigure}[b]{0.28\textwidth}
        \includegraphics[width=\textwidth]{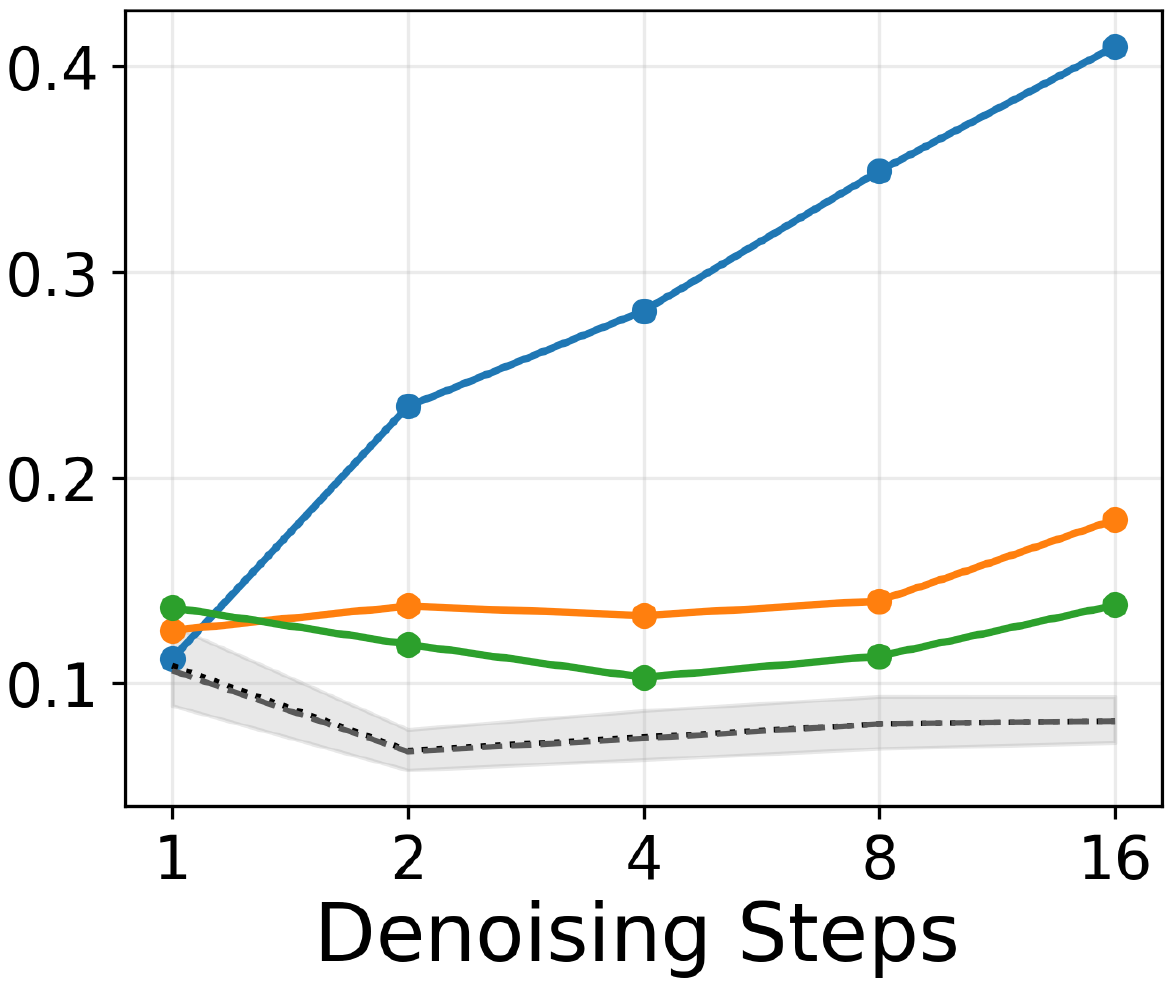}
        \caption{\dreamCoderInstruct}
        \label{fig:intro:b}
    \end{subfigure}
    \hfill
    \begin{subfigure}[b]{0.29\textwidth}
        \includegraphics[width=\textwidth]{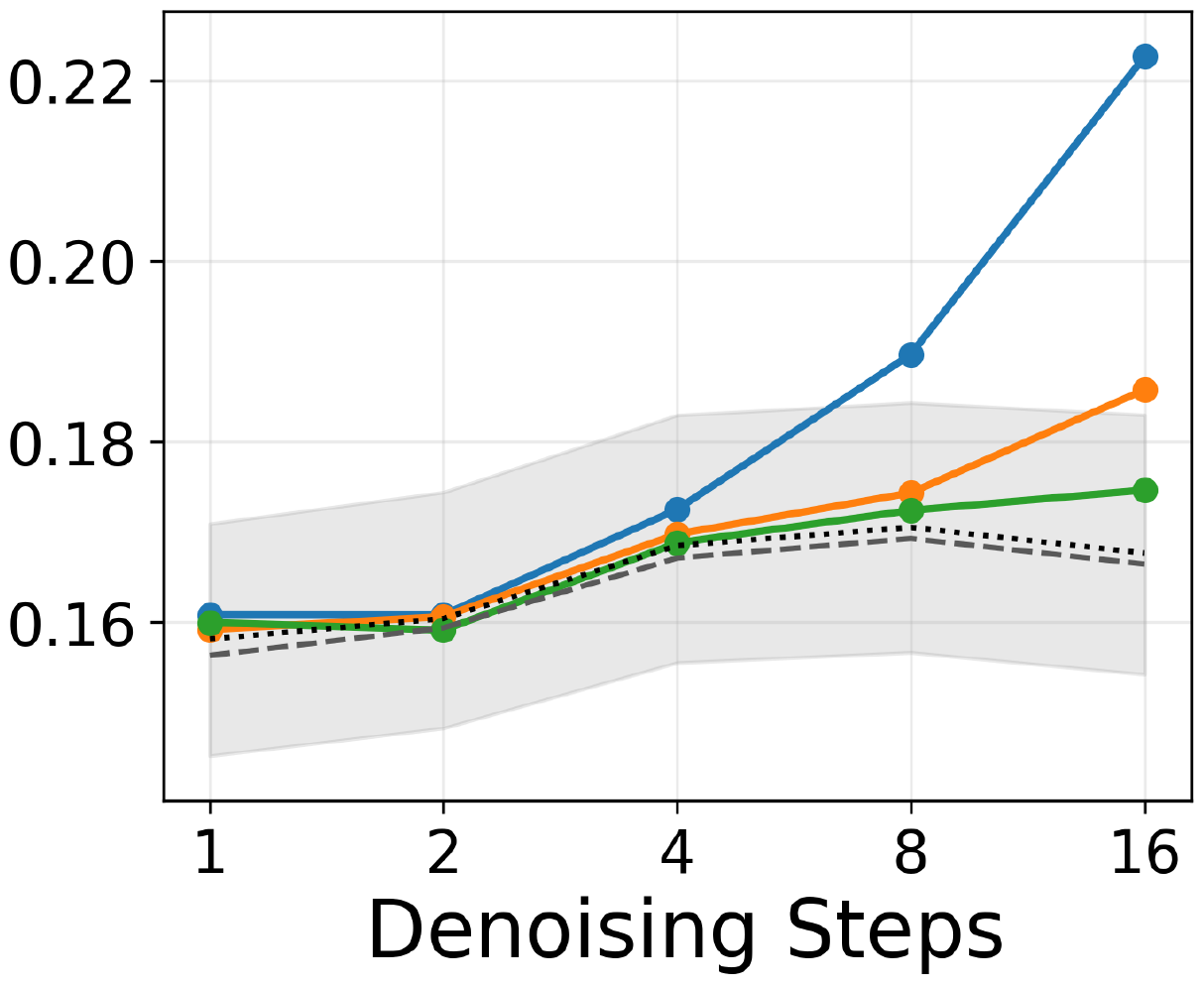}
        \caption{\dream}
        \label{fig:intro:c}
    \end{subfigure}
    \vspace{-2mm}
    \caption{\FigIntroBiasCaption\label{fig:intro}}
    \vspace{-4mm}
\end{figure}

\section{Introduction}
\label{sec:intro}

Masked Diffusion Language Models
(\mdlms)~\citep{sahoosimple24,niellada25,xieDreamCoder25,yeDream25}
are increasingly used for generating text, code, and other data, that may require their outputs to adhere to a specified format structure (syntax), \eg a JSON schema.

Unlike autoregressive Large Language Models (\llms), \mdlms decode by
repeatedly denoising a masked input; they first predict mask-free
token-categoricals over the masked positions, and then
retain only a scheduled subset of these predictions. Therefore, the
generation order depends on the sequence of these scheduled subsets of
positions and is not a left-to-right prefix like in \llms.
This creates a different kind of constrained decoding problem.  
\citet{sureshdingo25} introduced \dingo, an \mdlm algorithm for regular languages, which used maximum a posteriori (MAP) decoding over the chain-structured factorization of the model's automaton-constrained posterior. While \dingo is mode-seeking, \citet{dang2026constrained} proposed an algorithm for instead sampling (exactly) from the automaton-constrained posterior at each step using Forward-Filtering Backward-Sampling (\ffbs)~\citep{carterGibbs1994}.

We show that sampling locally from the automaton-constrained posterior
at each step, although exact, is still biased when composed across
steps, thus deviating from the globally constrained distribution.
%
%
The bias arises because the denoising process only retains a scheduled
subset of the sampled tokens and discards the rest when producing the
next intermediate state.
The resulting \emph{clamped-partition sum} is the total
constraint-valid probability mass obtained by fixing (or clamping) the
retained predictions to the successor state and marginalizing the
discarded ones under the current-state-conditioned \mdlm. At the next
step, the same positions (discarded in the previous step) are
reevaluated under the successor-state-conditioned \mdlm, producing a
new local-partition sum. The ratios of the clamped-partition sum to
the local-partition sum tilt the resulting path law relative to the
globally constrained one, which we model via the Doob
$\fDoobHarmonic{}$-transform~\citep{doobConditional1957} of the native
path law, conditioned on constraint-satisfaction. The tilt manifests
as \emph{trajectory bias}. Figure~\ref{fig:intro} shows the trajectory bias when sampling locally from the automaton-constrained posterior using a step-exact decoder, relative to the native decoder conditioned on constraint satisfaction, for different models.

Unlike locally constrained \llm decoding, which ignores future valid
mass~\citep{parkGrammarAligned24,dangMitigating26}, step-exact \mdlm
decoding accounts for this mass exactly, but under a denoiser frozen
at the current state. When these exact local steps are composed,
reconditioning the denoiser can change the valid mass and bias the
resulting distribution.
We correct the bias by introducing \tool, which uses
a \emph{Feynman--Kac correction} with the local and clamped-partition
sum ratios as the incremental potentials. The correction requires
quantities already computed by \ffbs of the step-exact sampler
proposed by \citet{dang2026constrained}, making it exactly and
efficiently computable.

\section{Background}
\label{sec:background}

In this section, we introduce notation used in the rest of the paper,
as well as background on \mdlms and twisted Sequential Monte Carlo.
We use $[\fPositionMax]$ to denote $\{1,2,\ldots,\fPositionMax\}$, the
set of natural numbers up to $\fPositionMax$; and
\smash{\footnotesize$\fIndicatorFunc{\cdot}$} to denote the standard
indicator function. For two $\fPositionMax$-length sequences
$\fDiffState{}$ and $\fDiffState{}'$, we write
$\fPositionEquiv{\fDiffState{}}{\fDiffState{}'}{\fRemaskingSet{}}$ if
$\fDiffState{}(i) = \fDiffState{}'(i)$ for every $i \in
\fRemaskingSet{}\subseteq[\fPositionMax]$.
For bounded products such as
$\prod\nolimits_{\fDiffusionStep<\fDiffusionStepMax}$, we always
assume $\fDiffusionStep$ starts from $0$.
Unnormalized quantities are generally denoted by a tilde, \eg $\fTargetProbUNorm$.

\subsection{Deterministic Finite Automata and Regular Languages}
\label{sec:background-dfa}

A deterministic finite automaton (\dfa) over a finite alphabet $\fAlpha$
%
is a tuple: 
\(
\fAutomtn\triangleq\fAutomtnTuple,
\)
where $\fStates$ is a finite set of states, $\fState{0}\in\fStates$ is the initial state, 
$\fStatesAcc \subseteq \fStates$ is the set of accepting states, and 
$\fTrans: \fStates \times \fAlpha \rightarrow \fStates$ is the transition function.
A word $\fWordW{}$ is accepted by a \dfa if the \dfa ends in an accepting
state after consuming $\fWordW{}$, \ie
$\fTransClosureFunc{\fState{0}}{\fWordW{}} \!\in\! \fStatesAcc$\footnotemark[3].

\footnotetext[3]{
$\fTransClosure$ is the \dfa's transition function $\fTrans$ extended
from symbols in $\fAlpha$ to words in
$\fAlphaStar$~\citep{farzanSound22}.
}

\begin{definition}[Regular Language]\label{def:regular-language}
    The language $\fLangSet$ recognized by a \dfa is the set of all words accepted by it: 
    \(
    \fLangSet\triangleq\{\fWordW{}\!\in\!\fAlphaStar \!\mid\! \fTransClosureFunc{\fState{0}}{\fWordW{}} \!\in\! \fStatesAcc\}.
    \)
    A language is regular if it is recognized by a \dfa.
\end{definition}




\subsection{Masked Diffusion Language Models}
\label{sec:background-mdlm}
\mdlms generate mask-free (or clean) token 
sequences from partially/fully masked sequences, by repeatedly unmasking the masked positions~\citep{sahoosimple24}.

Let $\fVocab$ be a finite vocabulary of clean tokens, 
$\fMask$ be the mask token, and $\fDiffStateSpace\!\triangleq\!\fVocab \cup \fSet{\fMask}$ the augmented token space. $\fDiffStateRV{}\!\in\!\fDiffStateSpace^{\fPositionMax}$ denotes a random state (an $\fPositionMax$-length token sequence) and $\fDiffState{}$ is its realization. The masked positions of $\fDiffState{}$ are given by:
\(
\fMaskFunc{\fDiffState{}} \!\triangleq\! \fSet{\fPosition \!\in\! [\fPositionMax] \vcentcolon \fDiffStateToken{}{\fPosition}{} \!=\! \fMask},
\)
 and 
$\fDiffStateToken{}{\fPosition}{}$ is the $\fPosition$-th token in $\fDiffState{}$.
Starting from a masked state $\fDiffStateRV{0}\!\in\!\fDiffStateSpace^{\fPositionMax}$, an \mdlm evolves as a Markov chain
$\fDiffStateRV{0} \to \fDiffStateRV{1} \to \cdots \to \fDiffStateRV{\fDiffusionStep} \to \cdots \to \fDiffStateRV{\fDiffusionStepMax}$
over a finite number of denoising steps $\fDiffusionStepMax$ to produce a clean state $\fDiffStateRV{\fDiffusionStepMax}\in\fVocab^{\fPositionMax}$.

The transition (Markov) kernels can be constructed from the three main components of \mdlms. 
For $\fDiffState{\fDiffusionStep}$ with $\fMaskFunc{\fDiffState{\fDiffusionStep}}\!\neq\!\varnothing$, the \emph{denoiser} outputs categoricals
\(
\fDiffCategorical{\fPosition}{\cdot}{\fDiffState{\fDiffusionStep}} \in \fSimplex{\fSize{\fVocab}-1}
\)
over $\fPosition \in \fMaskFunc{\fDiffState{\fDiffusionStep}}$, where $\fSimplex{\fSize{\fVocab}-1}$ is the $(\fSize{\fVocab}\!-\!1)$-simplex. The token predictions are conditionally independent given $\fDiffState{\fDiffusionStep}$, so the denoiser's posterior factorizes over $\fMaskFunc{\fDiffState{\fDiffusionStep}}$.
%
The \emph{scheduler} retains some of these factors by drawing a \emph{reveal set} $\fRemaskingSetRV{\fDiffusionStep+1}\!=\!\fRemaskingSet{}$ from some policy $\fDiffScheduler{\fDiffusionStep+1}{\cdot}{\fDiffState{\fDiffusionStep}}$ such that
$\fRemaskingSet{} \subseteq \fMaskFunc{\fDiffState{\fDiffusionStep}}$ and 
is non-empty.
The set $\fRemaskingSet{}$ is the subset of $\fDiffState{\fDiffusionStep}$'s masked positions that will be revealed for step $\fDiffusionStep\!+\!1$.
The policy $\fDiffScheduler{\fDiffusionStep+1}{\cdot}{\fDiffState{\fDiffusionStep}}$ can be uniform~\citep{niellada25}, confidence-based~\citep{wufastdllm26}, entropy-based~\citep{benhamuaccelerated25}, \etc
%
Lastly, the \emph{decoder} independently samples from $\fDiffCategorical{\fPosition}{\cdot}{\fDiffState{\fDiffusionStep}}$ for $\fPosition\!\in\! \fRemaskingSet{}$, and commits the sampled tokens (retaining $\fDiffState{\fDiffusionStep}$'s tokens in the remaining positions) to produce the next state $\fDiffState{\fDiffusionStep+1}$.
Since denoising is monotone unmasking, we also use $\fMaskDiffFunc{\fDiffState{\fDiffusionStep+1}}{\fDiffState{\fDiffusionStep}}\triangleq\fMaskFunc{\fDiffState{\fDiffusionStep}}\setminus\fMaskFunc{\fDiffState{\fDiffusionStep+1}}$ to denote the revealed set of positions between steps $\fDiffusionStep$ and $\fDiffusionStep\!+\!1$.

\begin{definition}[Native Joint Kernel]\label{def:joint-transition-kernel}
Transition from {\normalfont$\fDiffStateRV{\fDiffusionStep}\!=\!\fDiffState{\fDiffusionStep}$} to {\normalfont$\fDiffStateRV{\fDiffusionStep+1}\!=\!\fDiffState{\fDiffusionStep+1}$} and {\normalfont$\fRemaskingSetRV{\fDiffusionStep+1}\!=\!\fRemaskingSet{}$} is:
\begingroup
\setlength{\abovedisplayskip}{2pt}
\setlength{\belowdisplayskip}{0pt}
\normalfont
\newcommand{\fKernelUnderset}[2]{%
\underset{\text{#1}}{%
\vphantom{\fDiffScheduler{\fDiffusionStep+1}{\fRemaskingSet{\fDiffusionStep+1}}{\fDiffState{\fDiffusionStep}}
\fPositionEquiv{\fDiffState{\fDiffusionStep+1}}{\fDiffState{\fDiffusionStep}}{\neg\fRemaskingSet{}}
\prod\nolimits_{\fPosition \in \fRemaskingSet{}} \fDiffCategorical{\fPosition}{\fDiffStateToken{\fDiffusionStep+1}{\fPosition}{}}{\fDiffState{\fDiffusionStep}}}#2}}
\begin{equation}\label{eq:decoder}
    \fDiffKernel{\fDiffusionStep+1}{\fDiffState{\fDiffusionStep+1},\fRemaskingSet{}}{\fDiffState{\fDiffusionStep}} \triangleq
    \fKernelUnderset{drawn reveal set}{\fDiffScheduler{\fDiffusionStep+1}{\fRemaskingSet{}}{\fDiffState{\fDiffusionStep}}}\;\times\!\!
    \fKernelUnderset{\mdlms denoise by unmasking}{\fPositionEquiv{\fDiffState{\fDiffusionStep+1}}{\fDiffState{\fDiffusionStep}}{\neg\fRemaskingSet{}}}
    \!\!\times\;
    \fKernelUnderset{denoiser predictions over masks}{\prod\nolimits_{\fPosition \in \fRemaskingSet{}} \fDiffCategorical{\fPosition}{\fDiffStateToken{\fDiffusionStep+1}{\fPosition}{}}{\fDiffState{\fDiffusionStep}}}.
\end{equation}
\endgroup
\end{definition}
\mdlms can denoise only by unmasking which necessitates $\fPositionEquiv{\fDiffState{\fDiffusionStep+1}}{\fDiffState{\fDiffusionStep}}{\neg\fRemaskingSet{}}$ in Definition~\ref{def:joint-transition-kernel}\footnotemark[4]. 
\footnotetext[4]{We use $[\fPositionMax]$ as the universe for set operations of a state's positions: $\neg\fRemaskingSet{}\!\!\triangleq\![\fPositionMax]\!\setminus\!\fRemaskingSet{}, \neg\fMaskFunc{\fDiffState{}}\!\triangleq\![\fPositionMax]\!\setminus\!\fMaskFunc{\fDiffState{}},$ \etc}
The native kernel $\fDiffKernel{\fDiffusionStep+1}{\fDiffState{\fDiffusionStep+1}}{\fDiffState{\fDiffusionStep}}$ is obtained by marginalizing out $\fRemaskingSetRV{\fDiffusionStep+1}$ from Eq.~\eqref{eq:decoder}.
A denoising trajectory $\fDiffState{1:\fDiffusionStepMax}$ is then a sequence of states $(\fDiffState{1},\dots,\fDiffState{\fDiffusionStepMax})$ derived through repeated applications of the native kernels $\fDiffKernelSymbol_{\fDiffusionStep+1}$ on the corresponding intermediate states $\fDiffState{\fDiffusionStep}$ starting from the initial state $\fDiffState{0}$.

\begin{definition}[Native Path Law]\label{def:native-path-law}
The distribution of denoising trajectories conditioned on {\normalfont$\fDiffState{0}$}:
\begingroup
\setlength{\abovedisplayskip}{2pt}
\setlength{\belowdisplayskip}{0pt}
\normalfont
\begin{equation}\label{eq:native-path-law}
\fNativePathLaw{1:\fDiffusionStepMax}{\fDiffState{1:\fDiffusionStepMax}}{\fDiffState{0}}
\;\triangleq\;
\fProbMeasure{\fDiffStateRV{1}\!=\!\fDiffState{1},\dots,\fDiffStateRV{\fDiffusionStepMax}\!=\!\fDiffState{\fDiffusionStepMax}\mid\fDiffStateRV{0}\!=\!\fDiffState{0}}{}
=
\prod\nolimits_{\fDiffusionStep<\fDiffusionStepMax}
\fDiffKernel{\fDiffusionStep+1}{\fDiffState{\fDiffusionStep+1}}{\fDiffState{\fDiffusionStep}}.
\end{equation}
\endgroup
\end{definition}


\subsection{Twisted Sequential Monte Carlo}
\label{sec:background-smc}
%
Sequential Monte Carlo (\smc) is used for approximating trajectory distributions as it performs importance corrections incrementally while building the trajectory~\citep{Liu2004SMC}\footnotemark[5]. It approximates them using weighted particles, by propagating, reweighting, and resampling them as trajectories grow. It implements the \fkc model~\citep{DelMoral2004}.

\footnotetext[5]{Importance Sampling performs importance corrections only after a trajectory is built~\citep{Liu2004IS}.}

\begin{definition}[Feynman--Kac Model]\label{def:feynman-kac-model}
A Feynman--Kac model defines an unnormalized trajectory distribution $\fTargetProbUNorm_{0:\fDiffusionStep}$ for $\fDiffusionStep\!\leq\!\fDiffusionStepMax$, as a product of proposal kernels $\fFeynmanKacKernelSymbol_{}$ and  non-negative potentials $\fFeynmanKacPotentialSymbol_{}$:
\begingroup
\setlength{\abovedisplayskip}{4pt}
\setlength{\belowdisplayskip}{3pt}
\normalfont
\begin{equation}\label{eq:feynman-kac-model}
\fTargetProbUNormFunc{0:\fDiffusionStep}{\fDiffState{0:\fDiffusionStep}}
\;\triangleq\;
\fFeynmanKacKernelSymbol_{0}(\fDiffState{0})\,
\fFeynmanKacPotentialSymbol_{0}(\fDiffState{0})
\times
\prod\nolimits_{s<\fDiffusionStep}
\fFeynmanKacKernel{s+1}{\fDiffState{s+1}}{\fDiffState{s}}\,
\fFeynmanKacPotential{s+1}{\fDiffState{s}}{\fDiffState{s+1}}.
\end{equation}
\endgroup
\end{definition}

However, when the target includes only
a terminal weight,
standard \smc receives no guidance during intermediate steps.
Twisted \smc uses twist functions to move this information into incremental potentials, enabling particles to be guided as their trajectories grow~\citep{whiteley2014twisted}.

\begin{definition}[Twist-Constructed Potentials]\label{def:twist-constructed-potentials}
Consider a Markov path distribution
{\normalfont$\fProb(\fDiffState{0:\fDiffusionStep})\!=\!
\nu_{0}(\fDiffState{0})
\prod_{s<\fDiffusionStep}
\fDiffKernel{s+1}{\fDiffState{s+1}}{\fDiffState{s}}$}
and a non-negative terminal weight $h$.
Let $\fFeynmanKacTwistFunction{0:\fDiffusionStepMax}$ be twist functions such that
$\fFeynmanKacTwistFunction{\fDiffusionStep}\!>\!0$ for $\fDiffusionStep\!<\!\fDiffusionStepMax$ and
$\fFeynmanKacTwistFunction{\fDiffusionStepMax}\!=\!h$.
Given proposal kernels $\fFeynmanKacKernelSymbol$, the potentials are defined as:
\begingroup
\setlength{\abovedisplayskip}{4pt}
\setlength{\belowdisplayskip}{3pt}
\normalfont
\begin{equation}\label{eq:twisted-potentials}
\fFeynmanKacPotentialSymbol_{0}(\fDiffState{0})
\triangleq
\frac{\nu_{0}(\fDiffState{0})}
{\fFeynmanKacKernelSymbol_{0}(\fDiffState{0})}
\fFeynmanKacTwistFunction{0}(\fDiffState{0}),
\qquad
\fFeynmanKacPotential{\fDiffusionStep+1}{\fDiffState{\fDiffusionStep}}{\fDiffState{\fDiffusionStep+1}}
\triangleq
\frac{\fDiffKernel{\fDiffusionStep+1}{\fDiffState{\fDiffusionStep+1}}{\fDiffState{\fDiffusionStep}}}
{\fFeynmanKacKernel{\fDiffusionStep+1}{\fDiffState{\fDiffusionStep+1}}{\fDiffState{\fDiffusionStep}}}
\frac{\fFeynmanKacTwistFunction{\fDiffusionStep+1}(\fDiffState{\fDiffusionStep+1})}
{\fFeynmanKacTwistFunction{\fDiffusionStep}(\fDiffState{\fDiffusionStep})}.
\end{equation}
\endgroup
\end{definition}


Substituting the potentials from Eq.~\eqref{eq:twisted-potentials}
into Eq.~\eqref{eq:feynman-kac-model} gives
{\normalfont$\fTargetProbUNormFunc{0:\fDiffusionStep}{\fDiffState{0:\fDiffusionStep}}\!=\!\fProb(\fDiffState{0:\fDiffusionStep})\,
\fFeynmanKacTwistFunction{\fDiffusionStep}(\fDiffState{\fDiffusionStep})$}.
Therefore, the model targets
{\normalfont$\fProb(\fDiffState{0:\fDiffusionStepMax})h(\fDiffState{\fDiffusionStepMax})$} at $\fDiffusionStepMax$.
The intermediate twists do not alter it, but move its
information into the incremental potentials used to guide particles.
The ideal twist at $\fDiffState{\fDiffusionStep}$ is the expected terminal
weight under the remaining native transitions, \ie the exact lookahead or continuation probabilities.


\section{Step-Exact Does Not Mean Trajectory-Exact}
\label{sec:problem}

Constrained decoding in \mdlms operates along two axes: state positions $\fPosition\!\in\![\fPositionMax]$ and denoising steps $0\!\leq\!\fDiffusionStep\!<\!\fDiffusionStepMax$. At denoising step $\fDiffusionStep$, a step-exact decoder draws a clean sequence $\fDiffCleanStateRV{}\!\in\!\fVocab^{\fPositionMax}$
from the denoiser's automaton-constrained posterior (conditioned on $\fDiffStateRV{\fDiffusionStep}$) that satisfies the regular language constraint $\fLangSet$, \ie $\fDiffCleanStateRV{}\!\in\!\fLangSet$. Only $\fDiffCleanStateRV{}$'s positions in the reveal set $\fRemaskingSetRV{\fDiffusionStep+1}$ are committed to $\fDiffStateRV{\fDiffusionStep+1}$, while the remaining are discarded. A draw can therefore be exact at every step without the resulting trajectory $\fDiffStateRV{0:\fDiffusionStepMax}$ following the native path law (Definition~\ref{def:native-path-law}) conditioned on constraint-satisfaction $\fDiffStateRV{\fDiffusionStepMax}\!\in\!\fLangSet$.

Throughout the paper, we assume: (1)~complete reveal $\fMaskFunc{\fDiffState{\fDiffusionStepMax}}\!=\!\varnothing$ at the terminal denoising step $\fDiffusionStepMax$; and (2)~full support of the denoiser's predictions $\fDiffCategorical{\fPosition}{\fDiffCleanToken{}}{\fDiffState{}}>0$ for all $\fDiffCleanToken{}\in\fVocab$ and for all $\fPosition\in\fMaskFunc{\fDiffState{}}$.
We follow \citet{kooAutomata24} to lift the \dfa from symbols in $\fAlpha$ to tokens in the model's vocabulary $\fVocab$.

We now formally characterize the trajectory bias introduced by step-exact decoding. We start by defining the \emph{target path law}, \ie the native decoder's path law conditioned on constraint-satisfaction, using the Doob $\fDoobHarmonic{}$-transform (Section~\ref{sec:problem-native-decoder}). Next, we derive the transition kernels of the step-exact decoder (Section~\ref{sec:problem-automaton-constrained-decoder}), and show that the involved partition sums are lookaheads (valid-continuation probabilities) under a frozen denoiser, and thus step-exact decoding is a frozen Doob transform re-frozen at every step (Section~\ref{sec:method-frozen-denoiser}). Finally, we show that re-freezing is the trajectory bias (Section~\ref{sec:problem-trajectory-bias}).

\begin{definition}[Unbiased Constrained Decoder]\label{def:unbiased-decoder}
For an initial state {\normalfont$\fDiffState{0}$} and constraint {\normalfont$\fLangSet$}, let
{\normalfont$\Omega_{\fLangSet}(\fDiffState{0})$} be the support of native trajectories
{\normalfont$\fDiffState{1:\fDiffusionStepMax}$}
 satisfying {\normalfont$\fDiffState{\fDiffusionStepMax}\!\in\!\fLangSet$}.
A constrained decoder with path law
{\normalfont$q_{1:\fDiffusionStepMax}(\cdot\mid\fDiffState{0};\fLangSet)$}
is unbiased if its support is {\normalfont$\Omega_{\fLangSet}(\fDiffState{0})$} and, for every
{\normalfont$\fDiffState{1:\fDiffusionStepMax}',\fDiffState{1:\fDiffusionStepMax}''\!\in\!\Omega_{\fLangSet}(\fDiffState{0})$},
\begingroup
\setlength{\abovedisplayskip}{4pt}
\setlength{\belowdisplayskip}{4pt}
\normalfont
\begin{equation*}
\frac{q_{1:\fDiffusionStepMax}(\fDiffState{1:\fDiffusionStepMax}'\mid\fDiffState{0};\fLangSet)}
{q_{1:\fDiffusionStepMax}(\fDiffState{1:\fDiffusionStepMax}''\mid\fDiffState{0};\fLangSet)}
=
\frac{\fNativePathLaw{1:\fDiffusionStepMax}{\fDiffState{1:\fDiffusionStepMax}'}{\fDiffState{0}}}
{\fNativePathLaw{1:\fDiffusionStepMax}{\fDiffState{1:\fDiffusionStepMax}''}{\fDiffState{0}}},
\end{equation*}
\endgroup
\ie the constrained decoder does not alter the relative probabilities of constraint-valid trajectories.
\end{definition}

\subsection{\doob-Transformed Decoder}
\label{sec:problem-native-decoder}

In order to characterize trajectory bias, we first need the path law of a constrained decoder that is conditioned to satisfy $\fLangSet$, and is unbiased in the sense of Definition~\ref{def:unbiased-decoder}.
The native path law $\fNativePathLaw{1:\fDiffusionStepMax}{\cdot}{\fDiffStateRV{0}}$ is unconstrained and so does not satisfy $\fLangSet$ by itself. We obtain the unbiased constrained decoder by terminally conditioning $\fNativePathLaw{1:\fDiffusionStepMax}{\fDiffStateRV{1:\fDiffusionStepMax}}{\fDiffStateRV{0}}$ on satisfying $\fLangSet$, \ie $\fDiffStateRV{\fDiffusionStepMax}\!\in\!\fLangSet$; using the \doob $\fDoobHarmonic{}$-transform~\citep{doobConditional1957}. The \doob $\fDoobHarmonic{}$-transform conditions Markov processes to satisfy specific constraints, by tilting their transition kernels towards states compatible with highly-probable constraint-valid continuations that are given by the lookaheads $\fDoobHarmonic{\fDiffusionStep}(\fDiffState{}) \,\triangleq\, \fProbMeasure{\fDiffStateRV{\fDiffusionStepMax}\in\fLangSet\mid\fDiffStateRV{\fDiffusionStep}=\fDiffState{}}{}$.

\begin{definition}[\doob Kernel]
\label{def:doob-kernel}
For the native kernels $\fDiffKernelSymbol_{\fDiffusionStep+1}$ (Definition~\ref{def:joint-transition-kernel}) and their exact lookahead {\normalfont$\fDoobHarmonic{\fDiffusionStep}(\fDiffState{}),$} with {\normalfont$\fDoobHarmonic{\fDiffusionStep}(\fDiffState{})\!>\!0$} for all {\normalfont$\fDiffState{}$} in the support of {\normalfont$\fDiffKernelSymbol_{\fDiffusionStep+1}$},
their \doob $\fDoobHarmonic{}$-transform gives:
\begingroup
\setlength{\abovedisplayskip}{1pt}
\setlength{\belowdisplayskip}{1pt}
\normalfont
\begin{equation}\label{eq:doob-kernel}
    \fDoobKernel{\fDiffusionStep+1}{\fDiffState{\fDiffusionStep+1}}{\fDiffState{\fDiffusionStep};\fLangSet}
    \;\triangleq\;
    \fDiffKernel{\fDiffusionStep+1}{\fDiffState{\fDiffusionStep+1}}{\fDiffState{\fDiffusionStep}} \times
    \frac{\fDoobHarmonic{\fDiffusionStep+1}(\fDiffState{\fDiffusionStep+1})}
    {\fDoobHarmonic{\fDiffusionStep}(\fDiffState{\fDiffusionStep})}.
\end{equation}
\endgroup
\(
\displaystyle
\normalfont
\fDoobHarmonic{\fDiffusionStep}(\fDiffState{})
\triangleq
\sum\nolimits_{\fDiffState{}'\in\fDiffStateSpace^{\fPositionMax}} \fDiffKernel{\fDiffusionStep+1}{\fDiffState{}'}{\fDiffState{}}\;
\fDoobHarmonic{\fDiffusionStep+1}(\fDiffState{}'),
\)
is the exact lookahead of $\fDiffKernelSymbol_{\fDiffusionStep+1}$ and
\(
\normalfont
\fDoobHarmonic{\fDiffusionStepMax}(\fDiffState{})
\triangleq \fIndicatorFunc{\fDiffState{} \in \fLangSet}.
\)
\end{definition}

\begin{proposition}[\doob Path Law]
\label{prop:doob-path-law}
The \doob path law is the product of the \doob kernels:
\begingroup
\setlength{\abovedisplayskip}{2pt}
\setlength{\belowdisplayskip}{0pt}
\normalfont
\begin{equation}\label{eq:doob-path-law}
    \fDoobPathLaw{1:\fDiffusionStepMax}{\fDiffState{1:\fDiffusionStepMax}}{\fDiffState{0}}
    \;\;=\;\;
    \fNativePathLaw{1:\fDiffusionStepMax}{\fDiffState{1:\fDiffusionStepMax}}{\fDiffState{0}}
    \times
    \frac{\fIndicatorFunc{\fDiffState{\fDiffusionStepMax} \in \fLangSet}}
    {\fDoobHarmonic{0}(\fDiffState{0})}.
\end{equation}
\endgroup
\begin{proof}[Proof sketch]
The product of the \doob kernels $\fDoobKernelSymbol_{\fDiffusionStep+1}$
for $\fDiffusionStep\!<\!\fDiffusionStepMax$, gives the native path law from Definition~\ref{def:native-path-law}; and the $\fDoobHarmonic{}$-ratios telescope to $\fIndicatorFunc{\fDiffState{\fDiffusionStepMax} \in \fLangSet}/\fDoobHarmonic{0}(\fDiffState{0})$, thus giving Eq.~\eqref{eq:doob-path-law}.
\end{proof}
\end{proposition}

Importantly, the \doob path law $\fDoobPathLaw{1:\fDiffusionStepMax}{\fDiffStateRV{1:\fDiffusionStepMax}}{\fDiffState{0}}$ from Proposition~\ref{prop:doob-path-law} is unbiased as it prunes all trajectories with terminal state $\fDiffStateRV{\fDiffusionStepMax}\notin \fLangSet$, but does so without altering the relative probabilities of the valid trajectories under $\fNativePathLaw{1:\fDiffusionStepMax}{\fDiffStateRV{1:\fDiffusionStepMax}}{\fDiffState{0}}$. This follows from Eq.~\eqref{eq:doob-path-law}, where the \doob path law is proportional to $\fNativePathLaw{1:\fDiffusionStepMax}{\fDiffStateRV{1:\fDiffusionStepMax}}{\fDiffState{0}}$ up to a factor $\fDoobHarmonic{0}(\fDiffState{0})$ which is constant for a fixed initial state $\fDiffState{0}$.
The \doob terminal law can be obtained by marginalizing out the intermediate states $\fDiffStateRV{1},\dots,\fDiffStateRV{\fDiffusionStepMax-1}$ of $\fDoobPathLaw{1:\fDiffusionStepMax}{\fDiffStateRV{1:\fDiffusionStepMax}}{\fDiffState{0}}$. Let $\fDoobNormaliser(\fDiffState{0}) \triangleq \fDoobHarmonic{0}(\fDiffState{0}),$ then:
\begingroup
\setlength{\abovedisplayskip}{0pt}
\setlength{\belowdisplayskip}{2pt}
\normalfont
\begin{equation}\label{eq:terminal-distribution}
    \fDoobPathLaw{\fDiffusionStepMax}{\fDiffCleanState{}}{\fDiffState{0}}
    \;\;=\;\; 
    \fNativePathLaw{\fDiffusionStepMax}{\fDiffStateRV{\fDiffusionStepMax}\!=\!\fDiffCleanState{}}{\fDiffState{0}} \times \frac{\fIndicatorFunc{\fDiffCleanState{} \in \fLangSet}}{\fDoobNormaliser(\fDiffState{0})},
\end{equation}
\endgroup
gives the probability that the terminal \mdlm output is a clean sequence $\fDiffCleanState{}$, conditioned on the initial state $\fDiffState{0}$ and on the terminal output satisfying the regular language constraint.
When $\fDoobNormaliser(\fDiffState{0})$ is not too small, this distribution can be sampled efficiently using rejection sampling from the native decoder.

\subsection{Automaton-Constrained Step-Exact Decoder}
\label{sec:problem-automaton-constrained-decoder}

Unlike the \doob-transformed decoder which conditions the native path law on the final state $\fDiffStateRV{\fDiffusionStepMax}$ satisfying the constraint $\fLangSet$, the step-exact decoder constrains the denoiser's mean-field posterior at each denoising step to satisfy $\fLangSet$ and samples it exactly using \ffbs~\citep{dang2026constrained}.

Since the denoiser's published categoricals are conditionally independent given a state $\fDiffState{}$, its mean-field posterior $\fDenoiserPosterior{\cdot}{\fDiffState{}}$ factorizes over all the positions of $\fDiffState{}$ (respecting already committed positions) as:
\begingroup
\setlength{\abovedisplayskip}{4pt}
\setlength{\belowdisplayskip}{4pt}
\begin{equation}\label{eq:denoiser-posterior}
    \fDenoiserPosterior{\fDiffCleanStateRV{}\!=\!\fDiffCleanState{}}{\fDiffStateRV{}\!=\!\fDiffState{}} \;\;=\;\;
    \fPositionEquiv{\fDiffCleanState{}}{\fDiffState{}}{\neg\fMaskFunc{\fDiffState{}}} \times
    \prod\nolimits_{\fPosition \in \fMaskFunc{\fDiffState{}}} \fDiffCategorical{\fPosition}{\fDiffCleanStateToken{}{\fPosition}{}}{\fDiffState{}}.
\end{equation}
\endgroup
Sampling the mean-field posterior conditioned on a state $\fDiffState{}$ gives a clean sequence $\fDiffCleanState{}\in \fVocab^{\fPositionMax}$. This is because we retain $\fDiffState{}$'s tokens for non-masked positions while sampling from the mean-field's factors for masked positions. However, this clean sequence may not be in the regular language $\fLangSet$, \ie $\fDiffCleanState{}\notin\fLangSet$. Since the next intermediate state is derived from $\fDiffState{}$ and $\fDiffCleanState{}$, it may not satisfy the regular constraint either. Therefore, the step-exact decoder constrains the mean-field posterior to satisfy $\fLangSet$ at each step.

\begin{proposition}[Automaton-Constrained Posterior]\label{prop:automaton-constrained-posterior}
    Let the \dfa { \normalfont $\normalfont\fAutomtn\!=\!\fAutomtnTuple$} recognize $\fLangSet$, and {\normalfont$\fAutomtnStateChainRV\!\in\!\fStates^{\fPositionMax+1}$} be a random state chain of $\fAutomtn$, then the automaton-constrained posterior is:
\begingroup
\setlength{\abovedisplayskip}{4pt}
\setlength{\belowdisplayskip}{4pt}
\normalfont
\begin{equation}\label{eq:automaton-constrained-denoiser-posterior}
    \fDenoiserPosterior{\fDiffCleanStateRV{}\!=\!\fDiffCleanState{}, \fAutomtnStateChainRV\!=\!\fAutomtnStateChain}{\fDiffStateRV{}\!=\!\fDiffState{};\fLangSet} \;\;\propto\;\;
    \fDenoiserPosterior{\fDiffCleanState{}}{\fDiffState{}} \times
    \fAutomtnConstraint{\fAutomtn}{\fDiffCleanState{}}{\fAutomtnStateChain},
\end{equation}
\endgroup
where 
\(
\normalfont
\fAutomtnConstraint{\fAutomtn}{\fDiffCleanState{}}{\fAutomtnStateChain} 
\)
is the constraint imposed by $\fAutomtn$ and is given by:
\begingroup
\setlength{\abovedisplayskip}{4pt}
\setlength{\belowdisplayskip}{4pt}
\normalfont
\begin{equation}\label{eq:automaton-constraint}
\fAutomtnConstraint{\fAutomtn}{\fDiffCleanState{}}{\fAutomtnStateChain} \;\;=\;\;
\fIndicatorFunc{\fAutomtnStateChainToken{0}\!=\!\fState{0}}\times
\fIndicatorFunc{\fAutomtnStateChainToken{\fPositionMax}\!\in\! \fStatesAcc}\times
\prod\nolimits_{\fPosition \in [\fPositionMax]} \fIndicatorFunc{\fAutomtnStateChainToken{\fPosition}=\fTransFunc{\fAutomtnStateChainToken{\fPosition\!-\!1}}{\fDiffCleanStateToken{}{\fPosition}{}}}.
\end{equation}
\endgroup
\end{proposition}

The automaton-constrained posterior in Proposition~\ref{prop:automaton-constrained-posterior} is a Conditional Random Field (\crf)~\citep{laffertyConditional2001} with denoiser categoricals as unary potentials and the \dfa $\fAutomtn$'s transition function providing the compatibility factor for the transfer or pairwise potentials.
Since $\fAutomtn$ is deterministic, every constraint-satisfying $\fDiffCleanState{}\!\in\!\fLangSet$ corresponds to exactly one \dfa state chain for which $\fAutomtnConstraint{\fAutomtn}{\fDiffCleanState{}}{\fAutomtnStateChain}\!=\!1$.
Thus, marginalizing out $\fAutomtnStateChainRV$ gives
\(
\normalfont
\sum\nolimits_{\fAutomtnStateChain\in\fStates^{\fPositionMax+1}}\fAutomtnConstraint{\fAutomtn}{\fDiffCleanState{}}{\fAutomtnStateChain}
\!=\!\fIndicatorFunc{\fDiffCleanState{}\!\in\!\fLangSet}.
\)

\begin{definition}[Clamped Partition Sum]\label{def:clamped-partition-sum}
For a state {\normalfont$\fDiffState{}'$} reachable from {\normalfont$\fDiffState{}$}, \ie { \normalfont$\fPositionEquiv{\fDiffState{}'}{\fDiffState{}}{\neg\fMaskDiffFunc{\fDiffState{}'}{\fDiffState{}}}$} , the constraint-valid mass of the completions of {\normalfont$\fDiffState{}'$} under the denoiser conditioned on {\normalfont$\fDiffState{}$} is:
\begingroup
\setlength{\abovedisplayskip}{3pt}
\setlength{\belowdisplayskip}{4pt}
\normalfont
\begin{equation}\label{eq:ffbs-partition}
\fPartitionFuncClamp{}{\fDiffState{}'}{\fDiffState{}}
\;\triangleq\;
\sum\nolimits_{\fDiffCleanState{}\in\fVocab^{\fPositionMax}}
\fIndicatorFunc{\fDiffCleanState{}\in\fLangSet} \times
{\fPositionEquiv{\fDiffCleanState{}}{\fDiffState{}'}{\neg\fMaskFunc{\fDiffState{}'}}} \times
{\prod\nolimits_{\fPosition \in \fMaskFunc{\fDiffState{}'}}
\fDiffCategorical{\fPosition}{\fDiffCleanStateToken{}{\fPosition}{}}{\fDiffState{}}}.
\end{equation}
\endgroup
When { \normalfont$\fDiffState{}'=\fDiffState{}$}, the clamped partition sum {\normalfont$\fPartitionFuncClamp{}{\fDiffState{}}{\fDiffState{}}$} is called the local partition sum.
\end{definition}
The local partition sum
$\fPartitionFuncClamp{}{\fDiffState{}}{\fDiffState{}}$
is the normalizer of the constrained \crf in
Proposition~\ref{prop:automaton-constrained-posterior}, and scores the constraint-valid probability mass of the masked positions of $\fDiffState{}$ with the denoiser
conditioned on the same state $\fDiffState{}$, while the clamped partition sum $\fPartitionFuncClamp{}{\fDiffState{}'}{\fDiffState{}}$ scores it for the masked positions of $\fDiffState{}'$, but with the denoiser
conditioned on a different state $\fDiffState{}$.
Following \citet{dang2026constrained}, we use \ffbs both to sample exactly from
the constrained \crf and to compute the partition sums.

\begin{proposition}[\ffbsKernelProp Kernel]\label{prop:ffbs-kernel}
For the native kernels
\(
\normalfont
\fDiffKernelSymbol_{\fDiffusionStep+1}
\)
and partition sums {\normalfont$\fPartitionFuncClamp{}{\fDiffState{}}{\fDiffState{}}$} with
{\normalfont$\fPartitionFuncClamp{}{\fDiffState{}}{\fDiffState{}}>0$} for all {\normalfont$\fDiffState{}$} in the support of {\normalfont$\fDiffKernelSymbol_{\fDiffusionStep+1}$},
the \ffbsKernel transition kernel is given by:
\begingroup
\setlength{\abovedisplayskip}{4pt}
\setlength{\belowdisplayskip}{4pt}
\normalfont
\begin{align}
    \fFFBSKernel{\fDiffusionStep+1}{\fDiffState{\fDiffusionStep+1}}{\fDiffState{\fDiffusionStep};\fLangSet} \;=\;&
    \fDiffKernel{\fDiffusionStep+1}{\fDiffState{\fDiffusionStep+1}}{\fDiffState{\fDiffusionStep}} \times
    \frac{\fPartitionFuncClamp{}{\fDiffState{\fDiffusionStep+1}}{\fDiffState{\fDiffusionStep}}}{\fPartitionFuncClamp{}{\fDiffState{\fDiffusionStep}}{\fDiffState{\fDiffusionStep}}}\label{eq:ffbs-kernel}.
\end{align}
\endgroup
\begin{proof}[Proof sketch]
Substituting the constrained \crf into the native joint kernel factors out $\fDiffKernelSymbol_{\fDiffusionStep+1}$ and leaves $\fPartitionFuncClamp{}{\fDiffState{\fDiffusionStep+1}}{\fDiffState{\fDiffusionStep}}$ as the valid-completion sum. Formal proof in Appendix~\ref{app:proof-ffbs-kernel}.
\end{proof}
\end{proposition}

Interestingly, the \ffbsKernel kernel in Eq.~\eqref{eq:ffbs-kernel} looks structurally similar to the \doob kernel in Eq.~\eqref{eq:doob-kernel}, in the sense that both tilt the native kernel $\fDiffKernelSymbol_{\fDiffusionStep+1}$ by a ratio of successor to current constraint-valid probability masses, with the partition sums $Z$ of \ffbsKernel kernels paralleling the exact lookaheads $\fDoobHarmonic{}$ of the \doob kernels. We next show that this structural similarity is not a coincidence, but a consequence of the partition sums $Z$ being lookaheads, but of a different  Markov chain.


\subsection{Partition Sums are Lookaheads of a Frozen Denoiser Markov Chain}
\label{sec:method-frozen-denoiser}

The exact lookahead $\fDoobHarmonic{\fDiffusionStep}$ is intractable because the native Markov chain reconditions the denoiser on new states for each step. In contrast, a clamped partition sum $\fPartitionFuncClamp{}{\fDiffState{\fDiffusionStep+1}}{\fDiffState{\fDiffusionStep}}$ scores the valid completions of $\fDiffState{\fDiffusionStep+1}$ by conditioning the denoiser only once, at $\fDiffState{\fDiffusionStep}$. We
characterize their similarity by \emph{freezing} the denoiser at a reference state $\fDiffStateFrozen$. Then the frozen native joint kernel for $\fDiffState{\fDiffusionStep}$ reachable from $\fDiffStateFrozen$ is:
\begingroup
\setlength{\abovedisplayskip}{3pt}
\setlength{\belowdisplayskip}{3pt}
\normalfont
\begin{equation*}
\fDiffFrozenKernel{\fDiffusionStep+1}
{\fDiffState{\fDiffusionStep+1},\fRemaskingSet{\fDiffusionStep+1}}
{\fDiffState{\fDiffusionStep}}
{\fDiffStateFrozen}
\;\triangleq\;
\fDiffScheduler{\fDiffusionStep+1}
{\fRemaskingSet{\fDiffusionStep+1}}
{\fDiffState{\fDiffusionStep}} \times
\fPositionEquiv{\fDiffState{\fDiffusionStep+1}}
{\fDiffState{\fDiffusionStep}}
{\neg\fRemaskingSet{\fDiffusionStep+1}} \times
\prod\nolimits_{\fPosition\in\fRemaskingSet{\fDiffusionStep+1}}
\fDiffCategorical{\fPosition}
{\fDiffStateToken{\fDiffusionStep+1}{\fPosition}{}}
{\fDiffStateFrozen}.
\end{equation*}
\endgroup
Here, the denoiser's conditioning is fixed at $\fDiffStateFrozen$ instead of the current state $\fDiffState{\fDiffusionStep}$ as in Definition~\ref{def:joint-transition-kernel}. The scheduler is still evaluated at $\fDiffState{\fDiffusionStep}$. The frozen terminal law is independent of the order of unmasking since the predicted categoricals are fixed due to freezing the denoiser, yielding a tractable lookahead.

\begin{proposition}[Frozen Lookahead]\label{prop:frozen-lookahead}
The lookahead for the constraint-satisfying terminal law induced by the frozen native kernels, for any state {\normalfont$\fDiffState{\fDiffusionStep}$} reachable from {\normalfont$\fDiffStateFrozen$} is:
\begingroup
\setlength{\abovedisplayskip}{2pt}
\setlength{\belowdisplayskip}{2pt}
\normalfont
\begin{equation}\label{eq:frozen-lookahead}
\fFrozenDoobHarmonic{\fDiffusionStep}{\fDiffStateFrozen}{\fDiffState{\fDiffusionStep}}
=
\sum\nolimits_{\fDiffState{}\in\fDiffStateSpace^{\fPositionMax}}\;
\fPositionEquiv{\fDiffState{}}{\fDiffState{\fDiffusionStep}}{\neg\fMaskFunc{\fDiffState{\fDiffusionStep}}} \times
\fIndicatorFunc{\fDiffState{}\in\fLangSet} \times
\prod\nolimits_{\fPosition\in\fMaskFunc{\fDiffState{\fDiffusionStep}}}
\fDiffCategorical{\fPosition}
{\fDiffStateToken{}{\fPosition}{}}
{\fDiffStateFrozen}.
\end{equation}
\endgroup
Therefore,
\(
\normalfont
\fFrozenDoobHarmonic{\fDiffusionStep}{\fDiffState{\fDiffusionStep}}{\fDiffState{\fDiffusionStep}}
=
\fPartitionFuncClamp{}
{\fDiffState{\fDiffusionStep}}
{\fDiffState{\fDiffusionStep}}
\)
\;and\;
\(
\normalfont
\fFrozenDoobHarmonic{\fDiffusionStep+1}{\fDiffState{\fDiffusionStep}}{\fDiffState{\fDiffusionStep+1}}
=
\fPartitionFuncClamp{}
{\fDiffState{\fDiffusionStep+1}}
{\fDiffState{\fDiffusionStep}}
\).
\begin{proof}[Proof sketch]
Monotone unmasking and posterior sampling are independent of the scheduler, giving Eq.~\eqref{eq:frozen-lookahead}.
The two identities follow from Eq.~\eqref{eq:ffbs-partition} for $\fDiffStateFrozen\!=\!\fDiffState{\fDiffusionStep}$. Formal proof in Appendix~\ref{app:proof-frozen-lookahead}.
\end{proof}
\end{proposition}

From Proposition~\ref{prop:frozen-lookahead}, both partition sums in Eq.~\eqref{eq:ffbs-kernel} are thus lookaheads of the Markov chain frozen at the current state, which agrees with the native chain on the first transition, \ie when $\fDiffStateFrozen\!=\!\fDiffState{\fDiffusionStep}$.
\begin{proposition}[Re-Frozen Doob Transforms]
\label{prop:ffbs-frozen-doob}
Each \ffbsKernel kernel is the exact Doob transform of the frozen native kernels at its current state {\normalfont$\fDiffState{\fDiffusionStep}$}:
\begingroup
\setlength{\abovedisplayskip}{-6pt}
\setlength{\belowdisplayskip}{-1pt}
\normalfont
\begin{equation*}
\fFFBSKernel{\fDiffusionStep+1}
{\fDiffState{\fDiffusionStep+1}}
{\fDiffState{\fDiffusionStep};\fLangSet}
\;=\;
\fDiffKernel{\fDiffusionStep+1}
{\fDiffState{\fDiffusionStep+1}}
{\fDiffState{\fDiffusionStep}}\times
\frac{\fFrozenDoobHarmonic{\fDiffusionStep+1}{\fDiffState{\fDiffusionStep}}{\fDiffState{\fDiffusionStep+1}}}
{\fFrozenDoobHarmonic{\fDiffusionStep}{\fDiffState{\fDiffusionStep}}{\fDiffState{\fDiffusionStep}}}.
\end{equation*}
\endgroup
\end{proposition}

Proposition~\ref{prop:ffbs-frozen-doob} applies to the chain frozen at the current state. At the next step, however, the denoiser is re-frozen at the new state, replacing the expected successor 
\smash{$\fFrozenDoobHarmonic{\fDiffusionStep+1}{\fDiffState{\fDiffusionStep}}{\fDiffState{\fDiffusionStep+1}}$} with \smash{$\fFrozenDoobHarmonic{\fDiffusionStep+1}{\fDiffState{\fDiffusionStep+1}}{\fDiffState{\fDiffusionStep+1}}$}. It is this re-freezing that prevents the composition of the \ffbsKernel kernels from being the Doob transform of a single Markov chain. Since the partition sums are computed at each step, an intermediate state $\fDiffState{\fDiffusionStep+1}$ is thus scored twice, by $\fPartitionFuncClamp{}{\fDiffState{\fDiffusionStep+1}}{\fDiffState{\fDiffusionStep}}$ as a successor (\eg \raisebox{0.1ex}{\circledCaptionTiny{2}} in Figure~\ref{fig:trajectory-bias-choice}) and by $\fPartitionFuncClamp{}{\fDiffState{\fDiffusionStep+1}}{\fDiffState{\fDiffusionStep+1}}$ as the current state (\eg \raisebox{0.1ex}{\circledCaptionTiny{3}} in Figure~\ref{fig:trajectory-bias-choice}) at step $\fDiffusionStep\!+\!1$. We next show how this causes trajectory bias.

\subsection{Trajectory Bias}
\label{sec:problem-trajectory-bias}

\begin{figure*}[t]
    \centering
    \begin{subfigure}[t]{0.21\textwidth}
        \centering
        \resizebox{\linewidth}{!}{
            \begin{tikzpicture}[
    x=1cm,
    y=1cm,
    >={Stealth[length=1.7mm,width=1.2mm]},
    qnode/.style={circle, draw, thick, minimum size=6.5mm, inner sep=0pt,
                  font=\small, fill=yellow!10},
    qstart/.style={qnode, fill=blue!10, draw=blue!60!black},
    qaccept/.style={qnode, double, double distance=0.8pt,
                   fill=orange!15, draw=orange!75!black},
    qdead/.style={qnode, fill=gray!10, draw=gray!55, text=gray!65}
]
\path[use as bounding box] (-0.05,0.05) rectangle (3.15,4.05);
\node[font=\normalsize, anchor=west] at (0.18,1.2)
  {$\fLangSet=a^{*}b^{+}$};
\node[qstart]  (dfa0) at (0.80,2.58) {$q_0$};
\node[qaccept] (dfa1) at (2.30,2.58) {$q_1$};
\node[qdead]   (dfad) at (2.30,0.95) {$q_\varnothing$};
\draw[->, thick] (0.15,2.58) -- (dfa0);
\path (dfa0) edge[->, thick, loop above, looseness=6]
  node[font=\small, above] {$a$} (dfa0);
\path (dfa0) edge[->, thick]
  node[font=\small, above] {$b$} (dfa1);
\path (dfa1) edge[->, thick, loop above, looseness=6]
  node[font=\small, above] {$b$} (dfa1);
\path (dfa1) edge[->, thick, draw=gray!65]
  node[font=\small, right] {$a$} (dfad);
\path (dfad) edge[->, thick, loop below, looseness=6, draw=gray!65]
  node[font=\small, below] {$a,b$} (dfad);
\end{tikzpicture}
        }
        \caption{\dfa for $\fLangSet$.}
        \label{fig:trajectory-bias-dfa}
    \end{subfigure}
    \hfill
    \begin{subfigure}[t]{0.35\textwidth}
        \centering
        \resizebox{\linewidth}{!}{
            \begin{tikzpicture}[
    x=1cm,
    y=1cm,
    >={Stealth[length=1.7mm,width=1.2mm]},
    qnode/.style={circle, draw, thick, minimum size=6.5mm, inner sep=1pt,
                  font=\scriptsize, fill=yellow!10},
    qstart/.style={qnode, fill=blue!10, draw=blue!60!black},
    qaccept/.style={qnode, double, double distance=0.8pt,
                   fill=orange!15, draw=orange!75!black},
    factor/.style={rectangle, draw, fill=black, minimum size=2.1mm, inner sep=0pt},
    tokenvar/.style={rectangle, rounded corners=1pt, draw, thick,
                    minimum width=7mm, minimum height=5.3mm,
                    inner sep=1pt, font=\normalsize, fill=yellow!8},
    statecell/.style={rectangle, draw=black!70, minimum width=7mm,
                     minimum height=6mm, inner sep=0pt, font=\normalsize, fill=white},
    masked/.style={statecell, draw=gray!60, dashed, thick, fill=gray!3},
    probplot/.style={rectangle, rounded corners=1pt, draw=black!0,
                    fill=white, minimum width=10.5mm, minimum height=6.5mm,
                    inner sep=0pt}
]
\newcommand{\tbCrfProbBars}[6]{
  \node[probplot,#6] (#1) at (#2,#3) {};
  \begin{scope}[shift={(#2,#3)}, scale=1.3]
    \draw[black!25, line width=0.3pt] (-0.34,-0.12) -- (0.34,-0.12);
    \fill[magenta!80!black] (-0.30,-0.12) rectangle
      (-0.06,{-0.12 + 0.36*(#4)});
    \fill[green!70!black] (0.06,-0.12) rectangle
      (0.30,{-0.12 + 0.36*(#5)});
    \node[font=\normalsize, anchor=north] at (-0.18,-0.2) {$a$};
    \node[font=\normalsize, anchor=north] at (0.18,-0.13) {$b$};
  \end{scope}
}
\path[use as bounding box] (-0.05,0.05) rectangle (6.05,4.05);
\node[font=\normalsize, anchor=east] at (1.27,3.62)
  {$\fDiffState{0}$};
\node[masked]    (xt1) at (1.90,3.62) {$\bot$};
\node[masked]    (xt2) at (3.30,3.62) {$\bot$};
\node[statecell] (xt3) at (4.70,3.62) {$b$};
\node[qstart]  (cq0) at (1.20,2.78) {$\fAutomtnStateChain\bm{(0)}$};
\node[qnode]   (cq1) at (2.60,2.78) {$\fAutomtnStateChain\bm{(1)}$};
\node[qnode]   (cq2) at (4.00,2.78) {$\fAutomtnStateChain\bm{(2)}$};
\node[qaccept] (cq3) at (5.40,2.78) {$\fAutomtnStateChain\bm{(3)}$};
\node[factor] (f1) at (1.90,2.17) {};
\node[factor] (f2) at (3.30,2.17) {};
\node[factor] (f3) at (4.70,2.17) {};
\node[tokenvar] (y1) at (1.90,1.56) {$Y_1$};
\node[tokenvar] (y2) at (3.30,1.56) {$Y_2$};
\node[tokenvar] (y3) at (4.70,1.56) {$Y_3$};
\draw[thick] (cq0) -- (f1) -- (cq1) -- (f2) --
  (cq2) -- (f3) -- (cq3);
\foreach \i in {1,...,3} {
  \draw[thick] (f\i) -- (y\i);
}
\tbCrfProbBars{pb1}{1.90}{0.45}{0.5}{0.5}{}
\node at (2.58,0.8) {\circledCaption{1}};
\tbCrfProbBars{pb2}{3.30}{0.45}{0.5}{0.5}{}
\tbCrfProbBars{pb3}{4.70}{0.45}{0.0}{1.0}{}
\foreach \i in {1,...,3} {
  \draw[thick] (y\i.south) -- ($(pb\i.north)+(0,1mm)$);
}
\end{tikzpicture}
        }
        \caption{Automaton-constrained posterior.}
        \label{fig:trajectory-bias-posterior}
    \end{subfigure}
    \hfill
    \begin{subfigure}[t]{0.42\textwidth}
        \centering
        \resizebox{\linewidth}{!}{
            \begin{tikzpicture}[
    x=1cm,
    y=1cm,
    >={Stealth[length=1.7mm,width=1.2mm]},
    statecell/.style={rectangle, draw=black!70, minimum width=7mm,
                     minimum height=6mm, inner sep=0pt, font=\small, fill=white},
    committed/.style={statecell, draw=black, very thick, fill=black!8},
    masked/.style={statecell, draw=gray!60, dashed, thick, fill=gray!3},
    probplot/.style={rectangle, rounded corners=1pt, draw=black!0,
                    fill=white, minimum width=10.5mm, minimum height=5mm,
                    inner sep=-3pt}
]
\newcommand{\tbRecondProbBars}[6]{
  \node[probplot,#6] (#1) at (#2,#3) {};
  \begin{scope}[shift={(#2,#3)}, scale=1.3]
    \draw[black!25, line width=0.3pt] (-0.34,-0.12) -- (0.34,-0.12);
    \fill[magenta!80!black] (-0.30,-0.12) rectangle
      (-0.06,{-0.12 + 0.36*(#4)});
    \fill[green!70!black] (0.06,-0.12) rectangle
      (0.30,{-0.12 + 0.36*(#5)});
    \node[font=\normalsize, anchor=north] at (-0.18,-0.2) {$a$};
    \node[font=\normalsize, anchor=north] at (0.18,-0.13) {$b$};
  \end{scope}
}
\path[use as bounding box] (-0.65,0.05) rectangle (6.50,4.05);
\node[font=\normalsize, anchor=east] at (1.13,3.5)
  {$\fDiffState{1}^{a}$};
\node[masked]    (cx1) at (1.60,3.5) {$\bot$};
\node[committed] (cx2) at (3.00,3.5) {$a$};
\node[statecell] (cx3) at (4.40,3.5) {$b$};
\tbRecondProbBars{cold1}{1.60}{2.6}{0.5}{0.5}{draw=red!80, thick}
\node[font=\small, anchor=east, text=red!65!black]
  (coldscore) at (0.99,2.6) {$0.5$};
\tbRecondProbBars{cold2}{3.00}{2.6}{1.0}{0.0}{}
\tbRecondProbBars{cold3}{4.40}{2.6}{0.0}{1.0}{}
\node[font=\normalsize, anchor=west, align=left] at (4.90,3.5)
  {scored\\under\, $\fDiffState{0}$};
\node at (0.7,2.2) {\circledCaption{2}};
\node[font=\normalsize, anchor=east] at (1.13,1.3)
  {$\fDiffState{1}^{a}$};
\node[masked]    (cn1) at (1.60,1.3) {$\bot$};
\node[committed] (cn2) at (3.00,1.3) {$a$};
\node[statecell] (cn3) at (4.40,1.3) {$b$};
\tbRecondProbBars{cnew1}{1.60}{0.45}{0.9}{0.1}
  {draw=red!80, thick}
\node[font=\small, anchor=east, text=red!65!black]
  (cnewscore) at (0.99,0.45) {$0.9$};
\tbRecondProbBars{cnew2}{3.00}{0.45}{1.0}{0.0}{}
\tbRecondProbBars{cnew3}{4.40}{0.45}{0.0}{1.0}{}
\node[font=\normalsize, anchor=west, align=left] at (4.90,1.3)
  {rescored\\under\, $\fDiffState{1}^{a}$};
\node at (0.7,0.05) {\circledCaption{3}};
\draw[->, thick, red!65!black]
  (coldscore.west) to[out=180,in=180] (cnewscore.west);
\end{tikzpicture}
        }
        \caption{Successor $\fDiffState{1}^{a}$ rescoring at $\fDiffusionStep\!=\!0$ and $1$.}
        \label{fig:trajectory-bias-choice}
    \end{subfigure}
    \caption{\FigTrajectoryBiasCaption}
    \label{fig:trajectory-bias-overview}
\end{figure*}

We identify the tilt of the path law induced by composing the \ffbsKernel kernels, relative to the \doob path law, by multiplying them along a trajectory, yielding the native path law and the tilt factor:
\begingroup
\newcommand{\fAssociationBox}[2]{%
\tikzmarknode[inner xsep=1.3pt,inner ysep=0.8pt]{#1}{#2}}
\newcommand{\fAssociationFrac}[2]{%
\frac
{\raisebox{1.5pt}{$\textstyle #1$}}
{\raisebox{-1.5pt}{$\textstyle #2$}}}
\setlength{\abovedisplayskip}{4pt}
\setlength{\belowdisplayskip}{4pt}
\normalfont
\begin{equation*}
\prod\nolimits_{\fDiffusionStep<\fDiffusionStepMax}
\frac{\fPartitionFuncClamp{}{\fDiffState{\fDiffusionStep+1}}{\fDiffState{\fDiffusionStep}}}
{\fPartitionFuncClamp{}{\fDiffState{\fDiffusionStep}}{\fDiffState{\fDiffusionStep}}} 
\;\,=\;\,
\Bigg[
\fAssociationFrac
{\fAssociationBox{partition-pair-one-predecessor}
{\textcolor{lightblue}{\fPartitionFuncClamp{}{\fDiffState{1}}{\fDiffState{0}}}}}
{\fPartitionFuncClamp{}{\fDiffState{0}}{\fDiffState{0}}} 
\times
\fAssociationFrac
{\fAssociationBox{partition-pair-two-predecessor}
{\textcolor{lightblue}{\fPartitionFuncClamp{}{\fDiffState{2}}{\fDiffState{1}}}}}
{\fAssociationBox{partition-pair-one-successor}
{\textcolor{lightred}{\fPartitionFuncClamp{}{\fDiffState{1}}{\fDiffState{1}}}}}
\times
\fAssociationFrac
{\fAssociationBox{partition-pair-three-predecessor}
{\textcolor{lightblue}{\fPartitionFuncClamp{}{\fDiffState{3}}{\fDiffState{2}}}}}
{\fAssociationBox{partition-pair-two-successor}
{\textcolor{lightred}{\fPartitionFuncClamp{}{\fDiffState{2}}{\fDiffState{2}}}}}
\times
\tikzmarknode[inner sep=0pt]{partition-pair-ellipsis}{\cdots}
\times
\fAssociationFrac{\textcolor{lightgreen}{\fPartitionFuncClamp{}{\fDiffState{\fDiffusionStepMax}}{\fDiffState{\fDiffusionStepMax-1}}} }{\fAssociationBox{partition-pair-terminal-successor}
{\textcolor{lightred}{\fPartitionFuncClamp{}{\fDiffState{\fDiffusionStepMax-1}}{\fDiffState{\fDiffusionStepMax-1}}}}}
\Bigg]
\end{equation*}
\begin{tikzpicture}[remember picture,overlay]
\path[
    fill=lightpurple,
    fill opacity=0.05,
    draw=lightpurple!45,
    draw opacity=0.9,
    line width=0.6pt,
    rounded corners=1pt
]
(partition-pair-one-predecessor.north west)
-- (partition-pair-one-predecessor.north east)
-- ($(partition-pair-one-predecessor.south east)+(0,1.7pt)$)
-- ($(partition-pair-one-successor.north west)+(1.7pt,0)$)
-- (partition-pair-one-successor.north east)
-- (partition-pair-one-successor.south east)
-- (partition-pair-one-successor.south west)
-- ($(partition-pair-one-successor.north west)+(0,-1.7pt)$)
-- ($(partition-pair-one-predecessor.south east)+(-1.7pt,0)$)
-- (partition-pair-one-predecessor.south west)
-- cycle;
\path[
    fill=lightpurple,
    fill opacity=0.05,
    draw=lightpurple!45,
    draw opacity=0.9,
    line width=0.6pt,
    rounded corners=1pt
]
(partition-pair-two-predecessor.north west)
-- (partition-pair-two-predecessor.north east)
-- ($(partition-pair-two-predecessor.south east)+(0,1.7pt)$)
-- ($(partition-pair-two-successor.north west)+(1.7pt,0)$)
-- (partition-pair-two-successor.north east)
-- (partition-pair-two-successor.south east)
-- (partition-pair-two-successor.south west)
-- ($(partition-pair-two-successor.north west)+(0,-1.7pt)$)
-- ($(partition-pair-two-predecessor.south east)+(-1.7pt,0)$)
-- (partition-pair-two-predecessor.south west)
-- cycle;
\coordinate (partition-pair-three-cut)
at ($(partition-pair-three-predecessor.south east)+(5pt,-1.5pt)$);
\path[
    fill=lightpurple,
    fill opacity=0.05,
    draw=lightpurple!45,
    draw opacity=0.9,
    line width=0.6pt,
    rounded corners=1pt
]
(partition-pair-three-predecessor.north west)
-- (partition-pair-three-predecessor.north east)
-- ($(partition-pair-three-predecessor.south east)+(0,1.7pt)$)
-- ($(partition-pair-three-cut)+(0,1.3pt)$)
-- ($(partition-pair-three-cut)+(0,-1.3pt)$)
-- ($(partition-pair-three-predecessor.south east)+(-1.7pt,0)$)
-- (partition-pair-three-predecessor.south west)
-- cycle;
\coordinate (partition-pair-terminal-cut)
at ($(partition-pair-terminal-successor.north west)+(-5pt,1.5pt)$);
\path[
    fill=lightpurple,
    fill opacity=0.05,
    draw=lightpurple!45,
    draw opacity=0.9,
    line width=0.6pt,
    rounded corners=1pt
]
($(partition-pair-terminal-cut)+(0,1.3pt)$)
-- ($(partition-pair-terminal-successor.north west)+(1.7pt,0)$)
-- (partition-pair-terminal-successor.north east)
-- (partition-pair-terminal-successor.south east)
-- (partition-pair-terminal-successor.south west)
-- ($(partition-pair-terminal-successor.north west)+(0,-1.7pt)$)
-- ($(partition-pair-terminal-cut)+(0,-1.3pt)$)
-- cycle;
\end{tikzpicture}
\endgroup

The tilt factor is a product of ratios of partition sums.
Since all positions are revealed at $\fDiffusionStepMax$, Eq.~\eqref{eq:ffbs-partition} gives $\textcolor{lightgreen}{\fPartitionFuncClamp{}{\fDiffState{\fDiffusionStepMax}}{\fDiffState{\fDiffusionStepMax-1}}} \!=\! \fIndicatorFunc{\fDiffState{\fDiffusionStepMax}\!\in\!\fLangSet}$, the \doob terminal weight of Eq.~\eqref{eq:doob-path-law}. Each highlighted pair holds the two scores of an intermediate state $\fDiffState{\fDiffusionStep}$, scored as a successor $\textcolor{lightblue}{\fPartitionFuncClamp{}{\fDiffState{\fDiffusionStep}}{\fDiffState{\fDiffusionStep-1}}}$ and as the current state $\textcolor{lightred}{\fPartitionFuncClamp{}{\fDiffState{\fDiffusionStep}}{\fDiffState{\fDiffusionStep}}}$. We call the ratio $\fPartitionFuncClamp{}{\fDiffState{\fDiffusionStep}}{\fDiffState{\fDiffusionStep}}/\fPartitionFuncClamp{}{\fDiffState{\fDiffusionStep}}{\fDiffState{\fDiffusionStep-1}}$ the \emph{re-freezing ratio} of $\fDiffState{\fDiffusionStep}$. It measures how re-freezing changes the valid mass estimate of $\fDiffState{\fDiffusionStep}$'s completions. If it is greater than 1, then the valid mass of $\fDiffState{\fDiffusionStep}$ was underestimated, so paths through $\fDiffState{\fDiffusionStep}$ are under-weighted (and vice versa if less than 1):

\begin{proposition}[Re-Freezing Tilt]\label{prop:re-freezing-tilt}
The step-exact decoder samples the \doob path law tilted by the inverse re-freezing ratios of its intermediate states (for trajectories in common support):
\begingroup
\setlength{\abovedisplayskip}{3pt}
\setlength{\belowdisplayskip}{0pt}
\normalfont
\begin{align}
\fFFBSPathLaw{1:\fDiffusionStepMax}{\fDiffState{1:\fDiffusionStepMax}}{\fDiffState{0}}
=\;&
\fDoobPathLaw{1:\fDiffusionStepMax}{\fDiffState{1:\fDiffusionStepMax}}{\fDiffState{0}}
\times 
\frac{\fDoobNormaliser(\fDiffState{0})}
{\fPartitionFuncClamp{}{\fDiffState{0}}{\fDiffState{0}}}
\times 
\prod\nolimits_{\fDiffusionStep=1}^{\fDiffusionStepMax-1}
\frac{\fPartitionFuncClamp{}{\fDiffState{\fDiffusionStep}}{\fDiffState{\fDiffusionStep-1}}}
{\fPartitionFuncClamp{}{\fDiffState{\fDiffusionStep}}{\fDiffState{\fDiffusionStep}}}.
\label{eq:clamped-mass-tilt}
\end{align}
\endgroup
\begin{proof}[Proof sketch]
Substitute the \ffbsKernel and \doob kernels into the ratio of their path laws.
\end{proof}
\end{proposition}

Since $\fDoobNormaliser(\fDiffState{0})/\fPartitionFuncClamp{}{\fDiffState{0}}{\fDiffState{0}}$ is constant, the two path laws align if every re-freezing ratio equals $1$, \ie
\begingroup
\setlength{\abovedisplayskip}{3pt}
\setlength{\belowdisplayskip}{3pt}
\normalfont
\begin{equation*}
\qquad\qquad\qquad\qquad\qquad\qquad\qquad
\textcolor{lightblue}{\fPartitionFuncClamp{}{\fDiffState{t}}{\fDiffState{t-1}}}
=
\textcolor{lightred}{\fPartitionFuncClamp{}{\fDiffState{t}}{\fDiffState{t}}}, 
\qquad\qquad\quad
\text{for every }\; 0\!<\!\fDiffusionStep\!<\!\fDiffusionStepMax.
\end{equation*}
\endgroup
However, this generally does not hold, since unmasking tokens changes the denoiser's predictions for $\fMaskFunc{\fDiffState{\fDiffusionStep}}$, and this is structural to the working of \mdlms itself. 
Trajectory bias thus manifests due to scoring the same positions under two differently conditioned denoisers. Corollary~\ref{thm:trajectory-bias} gives the exact condition, which only requires the product of ratios to be constant.

\begin{corollary}[Trajectory Exactness]\label{thm:trajectory-bias}
The step-exact decoder is unbiased (Definition~\ref{def:unbiased-decoder}) if and only if the product of re-freezing ratios is constant over {\normalfont$\Omega_{\fLangSet}(\fDiffState{0})$}, \ie for every {\normalfont$\fDiffState{1:\fDiffusionStepMax}\in\Omega_{\fLangSet}(\fDiffState{0})$}:
\begingroup
\setlength{\abovedisplayskip}{3pt}
\setlength{\belowdisplayskip}{0pt}
\normalfont
\begin{equation}\label{eq:trajectory-exactness}
\prod\nolimits_{\fDiffusionStep=1}^{\fDiffusionStepMax-1}
\frac{\fPartitionFuncClamp{}{\fDiffState{\fDiffusionStep}}{\fDiffState{\fDiffusionStep}}}
{\fPartitionFuncClamp{}{\fDiffState{\fDiffusionStep}}{\fDiffState{\fDiffusionStep-1}}}
\;=\;
\frac{\fDoobNormaliser(\fDiffState{0})}
{\fPartitionFuncClamp{}{\fDiffState{0}}{\fDiffState{0}}}.
\end{equation}
\endgroup
\end{corollary}

\begin{corollary}[No Bias for Vacuous Constraints]\label{lemma:no-bias-for-vacuous-constraints}
For a vacuous regular constraint $\fLangSet\!=\!\fVocab^{\fPositionMax}$ (does not impose any restrictions), Eq.~\eqref{eq:trajectory-bias} holds for every $\fDiffusionStep\!<\!\fDiffusionStepMax$, and there is no trajectory bias.
\begin{proof}
Formal proof is in Appendix~\ref{app:proof-vacuous-constraints}.
\end{proof}
\end{corollary}

\begin{corollary}[No Bias for One-Step Reveal All]\label{lemma:no-bias-for-one-step-reveal}
When all masked positions in every initial state {\normalfont$\fDiffState{0}$} are revealed in a single step, \ie $\fDiffusionStepMax\!=\!1$, Eq.~\eqref{eq:trajectory-bias} holds and there is no trajectory bias.
\begin{proof}
Formal proof is in Appendix~\ref{app:proof-one-step-reveal}.
\end{proof}
\end{corollary}

The trajectory bias introduced by the tilt is correctable because: (1)~the tilt factorizes over the steps as shown in Eq.~\eqref{eq:clamped-mass-tilt} allowing it to be corrected incrementally at each step; and 
(2)~a re-freezing ratio can be computed efficiently and exactly at its step $\fDiffusionStep\!-\!1$ when $\fDiffState{\fDiffusionStep}$ is reached by computing the clamped partition sum $\fPartitionFuncClamp{}{\fDiffState{\fDiffusionStep}}{\fDiffState{\fDiffusionStep-1}}$ from the \ffbs messages that sampled $\fDiffState{\fDiffusionStep}$, and the next step's local partition sum $\fPartitionFuncClamp{}{\fDiffState{\fDiffusionStep}}{\fDiffState{\fDiffusionStep}}$ from the denoiser query at $\fDiffState{\fDiffusionStep}$ that is needed at the next step $\fDiffusionStep$ anyway.

\section{\tool: Automaton-Twisted Sequential Monte Carlo}
\label{sec:method}


Since the re-freezing tilt factorizes over steps
(Proposition~\ref{prop:re-freezing-tilt}), it can be corrected
incrementally at each step. We first show mathematically that the tilt
can be corrected via our proposed \fkc corrector, and then we describe
our \tool algorithm implementing it.

\subsection{\fkc Corrector for the Re-Freezing Tilt}
\label{sec:method-increment}

We use the tractable \ffbsKernel kernels from Proposition~\ref{prop:ffbs-frozen-doob} as the proposals:    
\(\fFeynmanKacKernel{\fDiffusionStep+1}
{\fDiffState{\fDiffusionStep+1}}
{\fDiffState{\fDiffusionStep}}
\!\triangleq\!
\fFFBSKernel{\fDiffusionStep+1}
{\fDiffState{\fDiffusionStep+1}}
{\fDiffState{\fDiffusionStep};\fLangSet}\),
and correct for the re-freezing ratio formalized in Section~\ref{sec:problem-trajectory-bias} with an exact twist.

\begin{definition}[\fkc Twist]\label{def:feynman-kac-twist-function}
The \fkc twist is given by:
\begingroup
\setlength{\abovedisplayskip}{2pt}
\setlength{\belowdisplayskip}{0pt}
\normalfont
\begin{equation*}
\fFeynmanKacTwistFunction{\fDiffusionStep}(\fDiffState{}) 
\;\triangleq\; 
\begin{cases}
\fPartitionFuncClamp{}{\fDiffState{}}{\fDiffState{}}, & \fDiffusionStep<\fDiffusionStepMax, \\
\fIndicatorFunc{\fDiffState{}\in\fLangSet}, & \fDiffusionStep=\fDiffusionStepMax
\end{cases}
\end{equation*}
\endgroup
\end{definition}
The ideal twist is the exact lookahead $\fDoobHarmonic{\fDiffusionStep}(\fDiffState{})$ from Definition~\ref{def:doob-kernel}, which requires future denoiser calls and is thus intractable.
The local partition sum $\fPartitionFuncClamp{}{\fDiffState{}}{\fDiffState{}}$ is the frozen lookahead \smash{$\fFrozenDoobHarmonic{\fDiffusionStep}{\fDiffState{}}{\fDiffState{}}$} from Proposition~\ref{prop:frozen-lookahead}, \ie a tractable surrogate of $\fDoobHarmonic{\fDiffusionStep}(\fDiffState{})$ from a single denoiser call at $\fDiffState{}$ at step $\fDiffusionStep$.
From Definition~\ref{def:feynman-kac-twist-function}, the \fkc twists are just the local partition sums at each step, thus they are exactly computable using the same \ffbs already used for sampling the proposals $\fFeynmanKacKernelSymbol_{\fDiffusionStep+1}$. The potentials $\fFeynmanKacPotentialSymbol_{\fDiffusionStep+1}$ of the \fkc model can then be constructed from the twists by Eq.~\eqref{eq:twisted-potentials} as:
\begingroup
\setlength{\abovedisplayskip}{4pt}
\setlength{\belowdisplayskip}{4pt}
\normalfont
\begin{equation*}
\fFeynmanKacPotentialSymbol_{0}(\fDiffState{0})
=
\fPartitionFuncClamp{}{\fDiffState{0}}{\fDiffState{0}},
\quad\quad
\fFeynmanKacPotential{\fDiffusionStep+1}{\fDiffState{\fDiffusionStep}}{\fDiffState{\fDiffusionStep+1}}
=
\frac{\fDiffKernel{\fDiffusionStep+1}{\fDiffState{\fDiffusionStep+1}}{\fDiffState{\fDiffusionStep}}}
{\fFeynmanKacKernel{\fDiffusionStep+1}{\fDiffState{\fDiffusionStep+1}}{\fDiffState{\fDiffusionStep}}}
\times
\frac{\fFeynmanKacTwistFunction{\fDiffusionStep+1}(\fDiffState{\fDiffusionStep+1})}
{\fFeynmanKacTwistFunction{\fDiffusionStep}(\fDiffState{\fDiffusionStep})}
\,=\,
\frac{\fPartitionFuncClamp{}{\fDiffState{\fDiffusionStep+1}}{\fDiffState{\fDiffusionStep+1}}}
{\fPartitionFuncClamp{}{\fDiffState{\fDiffusionStep+1}}{\fDiffState{\fDiffusionStep}}},
\end{equation*}
\endgroup
\ie the re-freezing ratio of $\fDiffState{\fDiffusionStep+1}$, for $\fDiffusionStep\!<\!(\fDiffusionStepMax\!-\!1)$, and $\fFeynmanKacPotentialSymbol_{\fDiffusionStepMax}\!=\!1$ since $\fPartitionFuncClamp{}{\fDiffState{\fDiffusionStepMax}}{\fDiffState{\fDiffusionStepMax-1}}=\fIndicatorFunc{\fDiffState{\fDiffusionStepMax}\in\fLangSet}$.

\begin{theorem}[Twist Fixes the Tilt]\label{thm:twist-fixes-the-tilt}
The \fkc twists correct the re-freezing tilt (Proposition~\ref{prop:re-freezing-tilt}), thus making the \fkc model target the \doob path law:
\begingroup
\setlength{\abovedisplayskip}{2pt}
\setlength{\belowdisplayskip}{0pt}
\normalfont
\begin{equation*}
\fFeynmanKacPotentialSymbol_{0}(\fDiffState{0})
\prod\nolimits_{\fDiffusionStep<\fDiffusionStepMax}
\fFeynmanKacKernel{\fDiffusionStep+1}{\fDiffState{\fDiffusionStep+1}}{\fDiffState{\fDiffusionStep}}\;
\fFeynmanKacPotential{\fDiffusionStep+1}{\fDiffState{\fDiffusionStep}}{\fDiffState{\fDiffusionStep+1}}
\;\;=\;\;
\fDoobPathLaw{1:\fDiffusionStepMax}{\fDiffState{1:\fDiffusionStepMax}}{\fDiffState{0}}\;
\fDoobNormaliser(\fDiffState{0})
\end{equation*}
\endgroup
\begin{proof}[Proof sketch]
Substitute for $\fFeynmanKacKernelSymbol_{\fDiffusionStep+1}$ and $\fFeynmanKacPotentialSymbol_{\fDiffusionStep+1}$ in Definition~\ref{def:feynman-kac-model}. Formal proof in Appendix~\ref{app:proof-twist-fixes-the-tilt}.
\end{proof}
\end{theorem}

In Figure~\ref{fig:trajectory-bias-overview}, weighting $\fDiffState{1}^{a}$ by its re-freezing ratio $9/5$ turns the step-exact choice $1/3\!:\!2/3$ into the \doob choice $9/19\!:\!10/19$.
\subsection{\tool Algorithm}
\label{sec:method-algorithm}

\begin{algorithm}[t]
\AlgorithmFontSize
\caption{\AlgorithmTwisterCaption}
\begin{algorithmic}[1]
    \setlength{\itemsep}{1.4pt}
    \Require \dfa $\fAutomtn$ recognizing $\fLangSet$;\;
    initial state $\fDiffState{0}$;\;
    number of denoising steps $\fDiffusionStepMax$;\;
    number of \smc particles $\fSMCParticleIndexMax$
    \State Query denoiser at $\fDiffState{0}$;\;
    \State $(\fDiffState{0}^{\fSMCParticleIndex},
    \fSMCParticleWeight{0}^{\fSMCParticleIndex})
    \gets(\fDiffState{0},1/\fSMCParticleIndexMax)$
    for all $\fSMCParticleIndex\in[\fSMCParticleIndexMax]$
    \For{$\fDiffusionStep=0,\ldots,\fDiffusionStepMax\!-\!1$}
        \Statex \AtsmcStage{lightpurple}{STEP-EXACT FFBS PROPAGATE}{8}
        \ParallelFor{$\fSMCParticleIndex=1,\ldots,\fSMCParticleIndexMax$}
            \State Sample
            $\fDiffState{\fDiffusionStep+1}^{\fSMCParticleIndex}
            \sim
            \fFeynmanKacKernel{\fDiffusionStep+1}{\cdot}
            {\fDiffState{\fDiffusionStep}^{\fSMCParticleIndex}}$
            \Comment{use cached \ffbs messages}
            \State $\widehat Z_{\fDiffusionStep+1}^{\fSMCParticleIndex}
            \gets
            \fPartitionFuncClamp{}
            {\fDiffState{\fDiffusionStep+1}^{\fSMCParticleIndex}}
            {\fDiffState{\fDiffusionStep}^{\fSMCParticleIndex}}$
            \Comment{predecessor clamped-partition sum}
            \If{$\fDiffusionStep<\fDiffusionStepMax\!-\!1$}
                \State Query denoiser at
                $\fDiffState{\fDiffusionStep+1}^{\fSMCParticleIndex}$;\;
                $Z_{\fDiffusionStep+1}^{\fSMCParticleIndex}
                \gets
                \fPartitionFuncClamp{}
                {\fDiffState{\fDiffusionStep+1}^{\fSMCParticleIndex}}
                {\fDiffState{\fDiffusionStep+1}^{\fSMCParticleIndex}}$
                \Comment{successor partition, cache \ffbs messages}
            \Else
                \State $Z_{\fDiffusionStepMax}^{\fSMCParticleIndex}
                \gets
                \fIndicatorFunc{
                \fDiffState{\fDiffusionStepMax}^{\fSMCParticleIndex}
                \in\fLangSet}$
            \EndIf
        \EndParallelFor
        \Statex \AtsmcStage{lightpurple}{RE-FREEZING CORRECTION}{6}
        \ParallelFor{$\fSMCParticleIndex=1,\ldots,\fSMCParticleIndexMax$}
            \State
            $\fFeynmanKacPotentialSymbol_{\fDiffusionStep+1}^{\fSMCParticleIndex}
            \gets
            {Z_{\fDiffusionStep+1}^{\fSMCParticleIndex}}/
            {\widehat Z_{\fDiffusionStep+1}^{\fSMCParticleIndex}}$
            \State $\fSMCParticleWeightUNorm{\fDiffusionStep+1}^{\fSMCParticleIndex}
            \gets
            \fSMCParticleWeight{\fDiffusionStep}^{\fSMCParticleIndex}
            \fFeynmanKacPotentialSymbol_{\fDiffusionStep+1}^{\fSMCParticleIndex}$
        \EndParallelFor
        \State $\fSMCParticleWeight{\fDiffusionStep+1}^{\fSMCParticleIndex}
        \gets
        \fSMCParticleWeightUNorm{\fDiffusionStep+1}^{\fSMCParticleIndex}/
        {\sum_{\fSMCParticleIndex'=1}^{\fSMCParticleIndexMax}
        \fSMCParticleWeightUNorm{\fDiffusionStep+1}^{\fSMCParticleIndex'}}$
        for all $\fSMCParticleIndex\in[\fSMCParticleIndexMax]$
    \EndFor
    \State \Return
    $(\fDiffState{\fDiffusionStepMax}^{1:\fSMCParticleIndexMax},
            \fSMCParticleWeight{\fDiffusionStepMax}^{1:\fSMCParticleIndexMax})$
\end{algorithmic}
\alglanguage{pseudocode}
\end{algorithm}

Algorithm~\ref{alg:twister} captures \tool's working. It requires a \dfa $\fAutomtn$ (lifted to the model's vocabulary $\fVocab$) that recognizes the regular language $\fLangSet$; an initial state $\fDiffState{0}$, \ie a token block containing the prompt and masked tokens; the number of denoising steps $\fDiffusionStepMax$; and the number of \smc particles $\fSMCParticleIndexMax$. 

\tool starts by querying the denoiser at $\fDiffState{0}$,
and caching its \ffbs messages. Following this, it initializes all $\fSMCParticleIndexMax$ \smc particles to $\fDiffState{0}$ with weight $1/\fSMCParticleIndexMax$. The cached \ffbs messages are used to sample from the \fkc proposal \smash{\footnotesize{$\fFeynmanKacKernelSymbol_{\fDiffusionStep+1}$}} (Proposition~\ref{prop:ffbs-frozen-doob}), and compute the predecessor clamped-partition sum \smash{\footnotesize$\widehat Z_{\fDiffusionStep+1}^{\fSMCParticleIndex}$}. For each propagated nonterminal particle, the denoiser is queried at the successor state to compute the successor partition sum \smash{\footnotesize$Z_{\fDiffusionStep+1}^{\fSMCParticleIndex}$} and cache the \ffbs messages for the next step.
The ratio
\smash{\footnotesize$\fFeynmanKacPotentialSymbol_{\fDiffusionStep+1}^{\fSMCParticleIndex}
=Z_{\fDiffusionStep+1}^{\fSMCParticleIndex}/
\widehat Z_{\fDiffusionStep+1}^{\fSMCParticleIndex}$} gives the exact \fkc potential required for fixing the re-freezing tilt (Section~\ref{sec:method-increment}). Each incoming particle weight is multiplied by this potential, and the resulting weights are normalized. Particles are jointly resampled with their caches and assigned uniform weight when resampling is triggered (described in Appendix~\ref{app:twister-full}).

\MyPara{Complexity}
One weighted-automaton pass costs
$\fBigO{\fPositionMax\fSize{\fStates}\fSize{\fVocab}}$, and so, the total automaton cost for $\fSMCParticleIndexMax$ \smc particles and $\fDiffusionStepMax$ denoising steps is $\fBigO{\fSMCParticleIndexMax\fDiffusionStepMax
\fPositionMax\fSize{\fStates}\fSize{\fVocab}}$, in addition to the denoiser evaluations. \tool is parallelizable over the step-exact \ffbs propagate and re-freezing correction phases, since the particles are conditionally independent until weight normalization and resampling. 

\subsection{Trajectory Bias Experiment}
\label{sec:method-eval}

We measure trajectory bias using regular language constraints defined over token classes. The number of token classes is an experiment configured parameter. Next, we rank the tokens by the probabilities assigned to them by the model conditioned on a fully masked input. We select the most probable tokens and then partition them equally across the token classes. Tokens outside this pool are assigned to the token class $x$. A model's generation can then be projected onto a small alphabet of symbols, \ie token classes $\{a,b,c,d,x\}$, \etc.

We use such token classes to build regular expressions which are then used to constrain the model's generation. We permute the token classes to ensure that the results do not depend on a particular token-to-class mapping. Trajectory bias is then measured by using either \cite{dang2026constrained}, $\gnardSmcFour$, or $\gnardSmcEight$ to sample from the automaton-constrained distribution subject to the regular expression; relative to the native decoder conditioned on constraint satisfaction approximated using rejection sampling (Eq.~\eqref{eq:terminal-distribution}).
We use pairwise total variation distance to measure the similarity between the two distributions.
We also compare independent splits of the rejection samples to estimate the finite-sample noise floor. Results are averaged with equal weight across partition configurations. Thus, deviations above the rejection noise floor measure trajectory bias rather than token-level validity or Monte Carlo error.
For each model and $\fDiffusionStepMax\in\{1,2,4,8,16\}$ denoising steps, we generate 40,000 native samples and 10,000 constrained samples at temperature~$1$ under random remasking. The one-step reveal for $\fDiffusionStepMax=1$ is an experimental setting of Corollary~\ref{lemma:no-bias-for-one-step-reveal}. We evaluate on the models \dream, \dreamInstruct, \dreamCoderInstruct, \llada, \lladaInstruct, and \kuleshov.

\section{Related Work}
\label{sec:related}

\MyPara{Constrained Decoding for Autoregressive Models}
%
Most recent approaches consist of maintaining a parser or automaton over the generated prefix and masking out invalid next tokens at each step~\citep{willardOutlines23, kooAutomata24, beurerDomino24, gengGrammar23, ugaresyncode24, dongXGrammar24}.
While this guarantees constraint satisfaction by construction, ~\citet{parkGrammarAligned24} showed that such locally constrained decoding distorts the model's distribution over valid sequences.
To address this, \citet{loula2025syntacticsemanticcontrollarge} approximate the globally constrained posterior with \smc, combining locally constrained decoding with incremental reweighting and resampling to correct the resulting bias.
\citet{dangMitigating26} further improve this setup by constructing stronger automaton-based proposals that encode future constraint satisfaction, yielding faster \smc convergence with fewer particles.
Like these approaches, we use \smc to target the model's globally constrained distribution rather than local constraint enforcement alone, but we focus on the less-explored domain of constrained decoding for \mdlms. Unlike locally constrained \llm decoding, which ignores future valid mass, step-exact \mdlm decoding accounts for this mass exactly, but under a denoiser frozen at the current state.

\MyPara{Constrained Decoding for Diffusion Models}
\citet{sureshdingo25} proposed \dingo as the first constrained decoder for \mdlms with formal guarantees.
\dingo uses dynamic programming over automaton states to select a constraint-satisfying block that maximizes the per-step mean-field probability.
\citet{dang2026constrained} further generalize \dingo's approach by replacing per-step MAP with exact sampling from the automaton-constrained mean-field posterior at each denoising step.
Like \dingo, their method guarantees that the constraint is satisfied by construction, but it further supports stochastic sampling under arbitrary remasking schedules.
Crucially, both the \dingo and \citet{dang2026constrained} still constrain generation \textit{only} at individual steps.
We show that such local exactness does not, however, preserve the model's relative probabilities over valid multi-step denoising trajectories.
We address this gap by introducing \tool as the first automaton-twisted \smc constrained decoder for \mdlms that targets the trajectory-exact globally constrained path law.
\citet{hasan2026discretefeynmankaccorrectors} apply Feynman--Kac \smc correctors to steer discrete diffusion sampling at inference time through temperature scaling or an external reward.
\citet{luo2026selfrewardingsequentialmontecarlo} likewise use \smc for \mdlms, but reweight particles by trajectory-level confidence to improve sample quality.
In contrast, \tool uses \smc to enforce formal constraints.

\section{Conclusion}
\label{sec:conclusion}

We showed that sampling locally from an \mdlm's automaton-constrained posterior, despite being exact at each denoising step, still tilts the distribution relative to the globally constrained distribution when composed across steps. We derived this trajectory bias as a product of local and clamped-partition sum ratios, which are a consequence of rescoring the constraint-valid completions of the same intermediate states, but under different
denoiser conditionings.
We showed that the tilt factorizes across the denoising steps, and thus can be corrected incrementally at each step.
We then introduced \tool, an automaton-twisted \smc constrained decoder that corrects this tilt by using the step-exact decoder as its proposal. Further, for regular-language constraints, its incremental \fkc potentials are exactly computable from partition sums already obtained during step-exact sampling. 

\subsection*{AI use statement}


We have not used generative AI tools for this work.

\subsection*{Reproducibility statement}


Our code and data will be made publicly available upon acceptance of the paper.

\bibliography{bib}
\bibliographystyle{iclr2027_conference}

\clearpage
\appendix
\startcontents[appendices]
\section*{Appendix Contents}
\printcontents[appendices]{}{1}{}

\clearpage
\section{Proof for \ffbsKernelProp Kernel}
\label{app:proof-ffbs-kernel}

\begingroup
\renewcommand{\thetheorem}{\ref{prop:ffbs-kernel}}
\begin{proposition}[\ffbsKernelProp Kernel]
For the native kernels
\(
\normalfont
\fDiffKernelSymbol_{\fDiffusionStep+1}
\)
and partition sums {\normalfont$\fPartitionFuncClamp{}{\fDiffState{}}{\fDiffState{}}$} with
{\normalfont$\fPartitionFuncClamp{}{\fDiffState{}}{\fDiffState{}}>0$} for all {\normalfont$\fDiffState{}$} in the support of {\normalfont$\fDiffKernelSymbol_{\fDiffusionStep+1}$},
the \ffbsKernel transition kernel is given by:
\begingroup
\setlength{\abovedisplayskip}{4pt}
\setlength{\belowdisplayskip}{4pt}
\normalfont
\begin{align*}
    \fFFBSKernel{\fDiffusionStep+1}{\fDiffState{\fDiffusionStep+1}}{\fDiffState{\fDiffusionStep};\fLangSet} \;=\;&
    \fDiffKernel{\fDiffusionStep+1}{\fDiffState{\fDiffusionStep+1}}{\fDiffState{\fDiffusionStep}} \times
    \frac{\fPartitionFuncClamp{}{\fDiffState{\fDiffusionStep+1}}{\fDiffState{\fDiffusionStep}}}{\fPartitionFuncClamp{}{\fDiffState{\fDiffusionStep}}{\fDiffState{\fDiffusionStep}}}.
\end{align*}
\endgroup
\begin{proof}
The joint \ffbsKernel kernel 
\(
\normalfont
\fFFBSKernel{\fDiffusionStep+1}{\fDiffState{}',\fRemaskingSet{}}{\fDiffState{\fDiffusionStep};\fLangSet}
\)
is given by:
\begingroup
\setlength{\abovedisplayskip}{4pt}
\setlength{\belowdisplayskip}{2pt}
\normalfont
\begin{flalign*}
\everymath={\displaystyle}
\normalfont
\phantom{\fDiffKernelSymbol_{\fDiffusionStep+1}}
&=\;
\fDiffScheduler{\fDiffusionStep+1}{\fRemaskingSet{}}{\fDiffState{\fDiffusionStep}}\;
\fPositionEquiv{\fDiffState{}'}{\fDiffState{\fDiffusionStep}}{\neg\fRemaskingSet{}}\;\times\;
\sum\nolimits_{\fDiffCleanState{},\fAutomtnStateChain}\;
\fDenoiserPosterior{\fDiffCleanState{},\fAutomtnStateChain}{\fDiffState{\fDiffusionStep};\fLangSet}\;
\fPositionEquiv{\fDiffState{}'}{\fDiffCleanState{}}{\fRemaskingSet{}}
&&\\
&=\;
\frac{
\fDiffScheduler{\fDiffusionStep+1}{\fRemaskingSet{}}{\fDiffState{\fDiffusionStep}}\;
\fPositionEquiv{\fDiffState{}'}{\fDiffState{\fDiffusionStep}}{\neg\fRemaskingSet{}}
}{\fPartitionFuncClamp{}{\fDiffState{\fDiffusionStep}}{\fDiffState{\fDiffusionStep}}}\;\times\;
\sum\nolimits_{\fDiffCleanState{},\fAutomtnStateChain}\;
\fDenoiserPosterior{\fDiffCleanState{}}{\fDiffState{\fDiffusionStep}}\;
\fAutomtnConstraint{\fAutomtn}{\fDiffCleanState{}}{\fAutomtnStateChain}\;
\fPositionEquiv{\fDiffState{}'}{\fDiffCleanState{}}{\fRemaskingSet{}}
&\text{by Eq.~\eqref{eq:automaton-constrained-denoiser-posterior}}&\\
&=\;
\fDiffKernel{\fDiffusionStep+1}{\fDiffState{}',\fRemaskingSet{}}{\fDiffState{\fDiffusionStep}}\;\times\;
\frac{
\fPartitionFuncClamp{}{\fDiffState{}'}{\fDiffState{\fDiffusionStep}}
}{
\fPartitionFuncClamp{}{\fDiffState{\fDiffusionStep}}{\fDiffState{\fDiffusionStep}}
}&\text{by Eq.~\eqref{eq:ffbs-partition}}&
\end{flalign*}
\endgroup
The \ffbsKernel kernel is obtained by marginalizing out the reveal set $\fRemaskingSetRV{}$ from the joint kernels.
\end{proof}
\end{proposition}
\endgroup
\section{Proof for Frozen Lookahead}
\label{app:proof-frozen-lookahead}

\begingroup
\renewcommand{\thetheorem}{\ref{prop:frozen-lookahead}}
\begin{proposition}[Frozen Lookahead]
The lookahead for the constraint-satisfying terminal law induced by the frozen native kernels, for any state {\normalfont$\fDiffState{\fDiffusionStep}$} reachable from {\normalfont$\fDiffStateFrozen$} is:
\begingroup
\setlength{\abovedisplayskip}{2pt}
\setlength{\belowdisplayskip}{2pt}
\normalfont
\begin{equation*}
\fFrozenDoobHarmonic{\fDiffusionStep}{\fDiffStateFrozen}{\fDiffState{\fDiffusionStep}}
=
\sum\nolimits_{\fDiffState{}\in\fDiffStateSpace^{\fPositionMax}}\;
\fPositionEquiv{\fDiffState{}}{\fDiffState{\fDiffusionStep}}{\neg\fMaskFunc{\fDiffState{\fDiffusionStep}}} \times
\fIndicatorFunc{\fDiffState{}\in\fLangSet} \times
\prod\nolimits_{\fPosition\in\fMaskFunc{\fDiffState{\fDiffusionStep}}}
\fDiffCategorical{\fPosition}
{\fDiffStateToken{}{\fPosition}{}}
{\fDiffStateFrozen}.
\end{equation*}
\endgroup
Therefore,
\(
\normalfont
\fFrozenDoobHarmonic{\fDiffusionStep}{\fDiffState{\fDiffusionStep}}{\fDiffState{\fDiffusionStep}}
=
\fPartitionFuncClamp{}
{\fDiffState{\fDiffusionStep}}
{\fDiffState{\fDiffusionStep}}
\)
\;and\;
\(
\normalfont
\fFrozenDoobHarmonic{\fDiffusionStep+1}{\fDiffState{\fDiffusionStep}}{\fDiffState{\fDiffusionStep+1}}
=
\fPartitionFuncClamp{}
{\fDiffState{\fDiffusionStep+1}}
{\fDiffState{\fDiffusionStep}}.
\).
\begin{proof}
Let $\mathcal{R}_{\fDiffusionStep}$ be the set of valid reveal-set sequences from $\fDiffusionStep$ to $\fDiffusionStepMax$. Then 
\(
\normalfont
\fProbMeasure{\fDiffStateRV{\fDiffusionStepMax}\!=\!\fDiffState{}
\!\mid\!\fDiffStateRV{\fDiffusionStep}\!=\!\fDiffState{\fDiffusionStep}}{\fDiffStateFrozen}
\)
is
\begingroup
\setlength{\abovedisplayskip}{4pt}
\setlength{\belowdisplayskip}{0pt}
\normalfont
\begin{flalign*}
\everymath={\displaystyle}
&=
\sum_{\fRemaskingSet{\fDiffusionStep+1:\fDiffusionStepMax}\in\mathcal{R}_{\fDiffusionStep}}
\prod_{s=\fDiffusionStep}^{\fDiffusionStepMax-1}
\fDiffFrozenKernel{s+1}
{\fDiffState{s+1},\fRemaskingSet{s+1}}
{\fDiffState{s}}
{\fDiffStateFrozen}
&&\\
&=
\fPositionEquiv{\fDiffState{}}{\fDiffState{\fDiffusionStep}}{\neg\fMaskFunc{\fDiffState{\fDiffusionStep}}}
\prod_{\fPosition\in\fMaskFunc{\fDiffState{\fDiffusionStep}}}
\fDiffCategorical{\fPosition}
{\fDiffStateToken{}{\fPosition}{}}
{\fDiffStateFrozen}
\times
\sum_{\fRemaskingSet{\fDiffusionStep+1:\fDiffusionStepMax}\in\mathcal{R}_{\fDiffusionStep}}
\prod_{s=\fDiffusionStep}^{\fDiffusionStepMax-1}
\fDiffScheduler{s+1}{\fRemaskingSet{s+1}}{\fDiffState{s}},
&&\\
&=
\fPositionEquiv{\fDiffState{}}{\fDiffState{\fDiffusionStep}}{\neg\fMaskFunc{\fDiffState{\fDiffusionStep}}}
\prod_{\fPosition\in\fMaskFunc{\fDiffState{\fDiffusionStep}}}
\fDiffCategorical{\fPosition}
{\fDiffStateToken{}{\fPosition}{}}
{\fDiffStateFrozen},
&\text{\hspace{-5cm}by $\displaystyle\sum_{r_{t+1:T}\in\mathcal{R}_{t}}\prod_{m=t}^{T-1}s_{m+1}(r_{m+1}\mid\fDiffState{m})=1$}&\\
\end{flalign*}
\endgroup

Therefore, the frozen lookahead is given by:

\begingroup
\setlength{\abovedisplayskip}{4pt}
\setlength{\belowdisplayskip}{2pt}
\normalfont
\begin{flalign*}
\everymath={\displaystyle}
\fFrozenDoobHarmonic{\fDiffusionStep}{\fDiffStateFrozen}{\fDiffState{\fDiffusionStep}}
&=
\sum_{\fDiffState{}\in\fDiffStateSpace^{\fPositionMax}}
\fIndicatorFunc{\fDiffState{}\in\fLangSet}\,
\fProbMeasure{\fDiffStateRV{\fDiffusionStepMax}=\fDiffState{}
\mid\fDiffStateRV{\fDiffusionStep}=\fDiffState{\fDiffusionStep}}{\fDiffStateFrozen}
&&\\
&=
\sum_{\fDiffState{}\in\fDiffStateSpace^{\fPositionMax}}
\fPositionEquiv{\fDiffState{}}{\fDiffState{\fDiffusionStep}}{\neg\fMaskFunc{\fDiffState{\fDiffusionStep}}}
\fIndicatorFunc{\fDiffState{}\in\fLangSet}
\prod_{\fPosition\in\fMaskFunc{\fDiffState{\fDiffusionStep}}}
\fDiffCategorical{\fPosition}
{\fDiffStateToken{}{\fPosition}{}}
{\fDiffStateFrozen},
&\text{giving Eq.~\eqref{eq:frozen-lookahead}}&\\
\end{flalign*}
\endgroup
Substituting $\fDiffStateFrozen=\fDiffState{\fDiffusionStep}$ in Eq.~\eqref{eq:ffbs-partition} gives
$\fFrozenDoobHarmonic{\fDiffusionStep}{\fDiffState{\fDiffusionStep}}{\fDiffState{\fDiffusionStep}}
\!=\!\fPartitionFuncClamp{}{\fDiffState{\fDiffusionStep}}{\fDiffState{\fDiffusionStep}}$
and
$\fFrozenDoobHarmonic{\fDiffusionStep+1}{\fDiffState{\fDiffusionStep}}{\fDiffState{\fDiffusionStep+1}}
\!=\!\fPartitionFuncClamp{}{\fDiffState{\fDiffusionStep+1}}{\fDiffState{\fDiffusionStep}}$.
\end{proof}
\end{proposition}
\endgroup

\section{Proofs for Trajectory-Bias Boundary Cases}
\label{app:trajectory-bias-proofs}

\subsection{Kernel-Level Trajectory Exactness}
\label{app:kernel-trajectory-exactness}

\begin{lemma}[Kernel-Level Trajectory Exactness]\label{lemma:kernel-trajectory-exactness} The \ffbsKernel path law composed of \ffbsKernel kernels {\normalfont$\fFFBSKernelSymbol_{\fDiffusionStep+1}$} equals that composed of the Doob kernels {\normalfont$\fDoobKernelSymbol_{\fDiffusionStep+1}$} if and only if, for every $\fDiffusionStep<\fDiffusionStepMax$, every state {\normalfont$\fDiffState{\fDiffusionStep}$} reachable with non-zero probability under the path law, with {\normalfont$\fPartitionFuncClamp{}{\fDiffState{\fDiffusionStep}}{\fDiffState{\fDiffusionStep}} \!>\! 0$} and  {\normalfont$\fDoobHarmonic{\fDiffusionStep}(\fDiffState{\fDiffusionStep})\!>\!0$}, and every {\normalfont$\fDiffState{\fDiffusionStep+1}$} in the support of {\normalfont$\fDiffKernelSymbol_{\fDiffusionStep+1}(\cdot\mid\fDiffState{\fDiffusionStep})$}:
\begingroup
\setlength{\abovedisplayskip}{4pt}
\setlength{\belowdisplayskip}{4pt}
\normalfont
\begin{equation}\label{eq:trajectory-bias}
    \frac{\fPartitionFuncClamp{}{\fDiffState{\fDiffusionStep+1}}{\fDiffState{\fDiffusionStep}}}
    {\fPartitionFuncClamp{}{\fDiffState{\fDiffusionStep}}{\fDiffState{\fDiffusionStep}}}
    \;\;=\;\;
    \frac{\fDoobHarmonic{\fDiffusionStep+1}(\fDiffState{\fDiffusionStep+1})}
    {\fDoobHarmonic{\fDiffusionStep}(\fDiffState{\fDiffusionStep})}
\end{equation}
\endgroup
\end{lemma}

\subsection{Vacuous Constraints}
\label{app:proof-vacuous-constraints}

\begingroup
\renewcommand{\thetheorem}{\ref{lemma:no-bias-for-vacuous-constraints}}
\begin{corollary}[No Bias for Vacuous Constraints]
For a vacuous constraining regular language $\fLangSet\!=\!\fVocab^{\fPositionMax}$, \ie constrains nothing, Eq.~\eqref{eq:trajectory-bias} holds for every $\fDiffusionStep\!<\!\fDiffusionStepMax$, and there is no trajectory bias.
\begin{proof}
Since $\fLangSet=\fVocab^{\fPositionMax}$, every terminal state $\fDiffStateRV{\fDiffusionStepMax}\in\fLangSet$, and so
\(
\normalfont
\fDoobHarmonic{\fDiffusionStepMax}(\fDiffState{})\!=\!1
\)
for all {\normalfont$\fDiffState{}\in\fVocab^{\fPositionMax}$}. By backward induction, suppose $\fDoobHarmonic{\fDiffusionStep+1}(\fDiffState{}')\!=\!1$ for every $\fDiffState{}'$.
\begingroup
\setlength{\abovedisplayskip}{4pt}
\setlength{\belowdisplayskip}{2pt}
\normalfont
\begin{flalign*}
\everymath={\displaystyle}
\normalfont
\fDoobHarmonic{\fDiffusionStep}(\fDiffState{})
&=
\sum\nolimits_{\fDiffState{}'} \fDiffKernel{\fDiffusionStep+1}{\fDiffState{}'}{\fDiffState{}}\;\; \fDoobHarmonic{\fDiffusionStep+1}(\fDiffState{}')
&\hspace*{-4cm}\text{by Definition~\ref{def:doob-kernel}}&\\[-2pt]
&=
\sum\nolimits_{\fDiffState{}'} \fDiffKernel{\fDiffusionStep+1}{\fDiffState{}'}{\fDiffState{}},
&\hspace*{-4cm}\text{by the backward-induction hypothesis}&\\
&= 1&&\\
\intertext{The local partition sum $\fPartitionFuncClamp{}{\fDiffState{}}{\fDiffState{}}$ is the normalizer of Eq.~\eqref{eq:automaton-constrained-denoiser-posterior} and so can be written as:}
\fPartitionFuncClamp{}{\fDiffState{}}{\fDiffState{}}
&=
\sum\nolimits_{\fDiffCleanState{}, \fAutomtnStateChain}
\fDenoiserPosterior{\fDiffCleanState{}}{\fDiffState{}}\;\,
\fAutomtnConstraint{\fAutomtn}{\fDiffCleanState{}}{\fAutomtnStateChain}&&\\
&=
\sum\nolimits_{\fDiffCleanState{}}
\fDenoiserPosterior{\fDiffCleanState{}}{\fDiffState{}},
&\hspace*{-4cm}\text{since $\fLangSet\!=\!\fVocab^{\fPositionMax}$ and $\fAutomtn$ is deterministic}&\\
&=
\sum\nolimits_{\fDiffCleanState{}}
\fPositionEquiv{\fDiffCleanState{}}{\fDiffState{}}{\neg\fMaskFunc{\fDiffState{}}}\;\,
\prod\nolimits_{\fPosition \in \fMaskFunc{\fDiffState{}}}
\fDiffCategorical{\fPosition}{\fDiffCleanStateToken{}{\fPosition}{}}{\fDiffState{}},
&\hspace*{-4cm}\text{from Eq.~\eqref{eq:denoiser-posterior}}&\\
&=
\prod\nolimits_{\fPosition \in \fMaskFunc{\fDiffState{}}}
\sum\nolimits_{\fDiffCleanToken{}\in\fVocab}
\fDiffCategorical{\fPosition}{\fDiffCleanToken{}}{\fDiffState{}},
&\hspace*{-4cm}\text{by\, $\sum\nolimits_{\fDiffCleanToken{}\in\fVocab}
\fDiffCategorical{\fPosition}{\fDiffCleanToken{}}{\fDiffState{}}=1$}&\\
&=
1&&\\
\intertext{The clamped partition sum $\fPartitionFuncClamp{}{\fDiffState{}'}{\fDiffState{}}$ is given as:}
\fPartitionFuncClamp{}{\fDiffState{}'}{\fDiffState{}}
&=
\sum\nolimits_{\fDiffCleanState{}}
\fIndicatorFunc{\fDiffCleanState{}\in\fLangSet}\;\,
\fPositionEquiv{\fDiffCleanState{}}{\fDiffState{}'}{\neg\fMaskFunc{\fDiffState{}'}}\;\,
\prod\nolimits_{\fPosition \in \fMaskFunc{\fDiffState{}'}}
\fDiffCategorical{\fPosition}{\fDiffCleanStateToken{}{\fPosition}{}}{\fDiffState{}},
&\text{from Eq.~\eqref{eq:ffbs-partition}}&\\
&=
\sum\nolimits_{\fDiffCleanState{}}
\fPositionEquiv{\fDiffCleanState{}}{\fDiffState{}'}{\neg\fMaskFunc{\fDiffState{}'}}\;\,
\prod\nolimits_{\fPosition \in \fMaskFunc{\fDiffState{}'}}
\fDiffCategorical{\fPosition}{\fDiffCleanStateToken{}{\fPosition}{}}{\fDiffState{}},
&\text{since $\fLangSet\!=\!\fVocab^{\fPositionMax}$}&\\
&=
\prod\nolimits_{\fPosition \in \fMaskFunc{\fDiffState{}'}}
\sum\nolimits_{\fDiffCleanToken{}\in\fVocab}
\fDiffCategorical{\fPosition}{\fDiffCleanToken{}}{\fDiffState{}},
&\hspace*{-1.4cm}\text{by\, $\sum\nolimits_{\fDiffCleanToken{}\in\fVocab}
\fDiffCategorical{\fPosition}{\fDiffCleanToken{}}{\fDiffState{}}=1$}&\\
&=
1&&
\end{flalign*}
\endgroup
Therefore
\(
\normalfont
\fPartitionFuncClamp{}{\fDiffState{}'}{\fDiffState{}}
\!=\!\fDoobHarmonic{\fDiffusionStep+1}(\fDiffState{}')
\!=\!\fPartitionFuncClamp{}{\fDiffState{}}{\fDiffState{}}
\!=\!\fDoobHarmonic{\fDiffusionStep}(\fDiffState{})
\!=\!1
\)
for all $\fDiffusionStep\!<\!\fDiffusionStepMax$.
\end{proof}
\end{corollary}
\endgroup

\subsection{One-Step Reveal All}
\label{app:proof-one-step-reveal}

\begingroup
\renewcommand{\thetheorem}{\ref{lemma:no-bias-for-one-step-reveal}}
\begin{corollary}[No Bias for One-Step Reveal All]
For a finite horizon $\fDiffusionStepMax\!=\!1$, where every initial realization {\normalfont$\fDiffState{0}$} is completely unmasked in one step, Eq.~\eqref{eq:trajectory-bias} holds and there is no trajectory bias.

\begin{proof}
Since $\fDiffusionStepMax\!=\!1$, we have
{\normalfont$\fDoobHarmonic{1}(\fDiffState{}')\!=\!\fIndicatorFunc{\fDiffState{}'\in\fLangSet}$}
from Definition~\ref{def:doob-kernel}, and so
{\normalfont$\fDoobHarmonic{0}(\fDiffState{0})$} is given by:
\begingroup
\setlength{\abovedisplayskip}{4pt}
\setlength{\belowdisplayskip}{2pt}
\normalfont
\begin{flalign*}
\everymath={\displaystyle}
\fDoobHarmonic{0}(\fDiffState{0})
&=
\sum\nolimits_{\fDiffState{}'} \fDiffKernel{1}{\fDiffState{}'}{\fDiffState{0}}\;\; \fDoobHarmonic{1}(\fDiffState{}')&&\\
&=
\sum\nolimits_{\fDiffState{}'}
\fDiffKernel{1}{\fDiffState{}'}{\fDiffState{0}}\;
\fIndicatorFunc{\fDiffState{}'\in\fLangSet},
&\hspace*{-4cm}\text{by\, $\fDoobHarmonic{1}(\fDiffState{}')\!=\!\fIndicatorFunc{\fDiffState{}'\in\fLangSet}$}&\\
\intertext{The clamped partition sum $\fPartitionFuncClamp{}{\fDiffState{}'}{\fDiffState{0}}$ is given by Eq.~\eqref{eq:ffbs-partition} as:}
\fPartitionFuncClamp{}{\fDiffState{}'}{\fDiffState{0}}
&=
\sum\nolimits_{\fDiffCleanState{}}
\fIndicatorFunc{\fDiffCleanState{}\in\fLangSet}\;\,
\fPositionEquiv{\fDiffCleanState{}}{\fDiffState{}'}{\neg\fMaskFunc{\fDiffState{}'}}\;\,
\prod\nolimits_{\fPosition \in \fMaskFunc{\fDiffState{}'}}
\fDiffCategorical{\fPosition}{\fDiffCleanStateToken{}{\fPosition}{}}{\fDiffState{0}}&&\\
&=
\sum\nolimits_{\fDiffCleanState{}}
\fIndicatorFunc{\fDiffCleanState{}\in\fLangSet}\;\,
\fIndicatorFunc{\fDiffState{}'\!=\!\fDiffCleanState{}},
&\hspace*{-3.4cm}\text{since \,$\fMaskFunc{\fDiffState{}'}\!=\!\varnothing$\, for\, $\fDiffusionStepMax\!=\!1$}&\\
&=
\fIndicatorFunc{\fDiffState{}'\in\fLangSet}&&\\
\intertext{The local partition sum $\fPartitionFuncClamp{}{\fDiffState{0}}{\fDiffState{0}}$ is the normalizer of Eq.~\eqref{eq:ffbs-kernel} and so can be written as:}
\fPartitionFuncClamp{}{\fDiffState{0}}{\fDiffState{0}}
&=
\sum\nolimits_{\fDiffState{}'}
\fDiffKernel{1}{\fDiffState{}'}{\fDiffState{0}}\;\;
\fPartitionFuncClamp{}{\fDiffState{}'}{\fDiffState{0}}&&\\
&=
\sum\nolimits_{\fDiffState{}'}
\fDiffKernel{1}{\fDiffState{}'}{\fDiffState{0}}\;\,
\fIndicatorFunc{\fDiffState{}'\in\fLangSet},
&\hspace*{-3.4cm}\text{by\, $\fPartitionFuncClamp{}{\fDiffState{}'}{\fDiffState{0}}\!=\!\fIndicatorFunc{\fDiffState{}'\in\fLangSet}$}&
\end{flalign*}
\endgroup
Since
\(
\normalfont
\fPartitionFuncClamp{}{\fDiffState{}'}{\fDiffState{0}}
\!=\!\fDoobHarmonic{1}(\fDiffState{}')
\)
and
\(
\normalfont
\fPartitionFuncClamp{}{\fDiffState{0}}{\fDiffState{0}}
\!=\!\fDoobHarmonic{0}(\fDiffState{0}),
\).
\end{proof}
\end{corollary}
\endgroup
\section{Proof for Twist Fixes the Tilt}
\label{app:proof-twist-fixes-the-tilt}

\begingroup
\renewcommand{\thetheorem}{\ref{thm:twist-fixes-the-tilt}}
\begin{theorem}[Twist Fixes the Tilt]
The \fkc twists correct the re-freezing tilt (Proposition~\ref{prop:re-freezing-tilt}), thus making the step-exact decoder path law target the \doob path law:
\begingroup
\setlength{\abovedisplayskip}{2pt}
\setlength{\belowdisplayskip}{0pt}
\normalfont
\begin{equation*}
\fFeynmanKacPotentialSymbol_{0}(\fDiffState{0})
\prod\nolimits_{\fDiffusionStep<\fDiffusionStepMax}
\fFeynmanKacKernel{\fDiffusionStep+1}{\fDiffState{\fDiffusionStep+1}}{\fDiffState{\fDiffusionStep}}\;
\fFeynmanKacPotential{\fDiffusionStep+1}{\fDiffState{\fDiffusionStep}}{\fDiffState{\fDiffusionStep+1}}
\;\;=\;\;
\fDoobPathLaw{1:\fDiffusionStepMax}{\fDiffState{1:\fDiffusionStepMax}}{\fDiffState{0}}\;
\fDoobNormaliser(\fDiffState{0}).
\end{equation*}
\endgroup
\begin{proof}
Substituting the twist-constructed potentials from Eq.~\eqref{eq:twisted-potentials} into the left-hand side gives:
\begingroup
\setlength{\abovedisplayskip}{4pt}
\setlength{\belowdisplayskip}{2pt}
\normalfont
\begin{flalign*}
\everymath={\displaystyle}
&\fFeynmanKacPotentialSymbol_{0}(\fDiffState{0})
\prod\nolimits_{\fDiffusionStep<\fDiffusionStepMax}
\fFeynmanKacKernel{\fDiffusionStep+1}{\fDiffState{\fDiffusionStep+1}}{\fDiffState{\fDiffusionStep}}\;
\fFeynmanKacPotential{\fDiffusionStep+1}{\fDiffState{\fDiffusionStep}}{\fDiffState{\fDiffusionStep+1}}
&&\\
&=\;
\fFeynmanKacTwistFunction{0}(\fDiffState{0})
\prod\nolimits_{\fDiffusionStep<\fDiffusionStepMax}
\fDiffKernel{\fDiffusionStep+1}{\fDiffState{\fDiffusionStep+1}}{\fDiffState{\fDiffusionStep}}\;
\frac{\fFeynmanKacTwistFunction{\fDiffusionStep+1}(\fDiffState{\fDiffusionStep+1})}
{\fFeynmanKacTwistFunction{\fDiffusionStep}(\fDiffState{\fDiffusionStep})}
&&\\
&=\;
\prod\nolimits_{\fDiffusionStep<\fDiffusionStepMax}
\fDiffKernel{\fDiffusionStep+1}{\fDiffState{\fDiffusionStep+1}}{\fDiffState{\fDiffusionStep}}\;
\fFeynmanKacTwistFunction{\fDiffusionStepMax}(\fDiffState{\fDiffusionStepMax})
&\text{by telescoping the twists}&\\
&=\;
\fNativePathLaw{1:\fDiffusionStepMax}{\fDiffState{1:\fDiffusionStepMax}}{\fDiffState{0}}\;
\fIndicatorFunc{\fDiffState{\fDiffusionStepMax}\in\fLangSet}
&\text{by Definition~\ref{def:feynman-kac-twist-function}}&\\
&=\;
\fDoobPathLaw{1:\fDiffusionStepMax}{\fDiffState{1:\fDiffusionStepMax}}{\fDiffState{0}}\;
\fDoobNormaliser(\fDiffState{0})
&\text{by Eq.~\eqref{eq:doob-path-law}}&
\end{flalign*}
\endgroup
\end{proof}
\end{theorem}
\endgroup

\section{\tool Algorithm (With Adaptive Resampling)}
\label{app:twister-full}

\begin{algorithm}[t]
    \AlgorithmFontSize
    \caption{\AlgorithmTwisterFullCaption}
    \begin{algorithmic}[1]
        \setlength{\itemsep}{1.4pt}
        \Require \dfa $\fAutomtn$ recognizing $\fLangSet$;\;
        initial state $\fDiffState{0}$;\;
        number of denoising steps $\fDiffusionStepMax$
        \Require number of \smc particles $\fSMCParticleIndexMax$;\;
        minimum effective sample size $\mathrm{ESS}_{\mathrm{min}}\in[0,1]$
        \State Query the denoiser at $\fDiffState{0}$
        \State $(\fDiffState{0}^{\fSMCParticleIndex},
        \fSMCParticleWeight{0}^{\fSMCParticleIndex})
        \gets(\fDiffState{0},1/\fSMCParticleIndexMax)$
        for all $\fSMCParticleIndex\in[\fSMCParticleIndexMax]$
        \For{$\fDiffusionStep=0,\ldots,\fDiffusionStepMax\!-\!1$}
            \Statex \AtsmcStage{lightpurple}{STEP-EXACT FFBS PROPAGATE}{8}
            \ParallelFor{$\fSMCParticleIndex=1,\ldots,\fSMCParticleIndexMax$}
                \State Sample
                $\fDiffState{\fDiffusionStep+1}^{\fSMCParticleIndex}
                \sim
                \fFeynmanKacKernel{\fDiffusionStep+1}{\cdot}
                {\fDiffState{\fDiffusionStep}^{\fSMCParticleIndex}}$
                \Comment{use cached \ffbs messages}
                \State $\widehat Z_{\fDiffusionStep+1}^{\fSMCParticleIndex}
                \gets
                \fPartitionFuncClamp{}
                {\fDiffState{\fDiffusionStep+1}^{\fSMCParticleIndex}}
                {\fDiffState{\fDiffusionStep}^{\fSMCParticleIndex}}$
                \Comment{predecessor clamped-partition sum}
                \If{$\fDiffusionStep<\fDiffusionStepMax\!-\!1$}
                    \State Query the denoiser at
                    $\fDiffState{\fDiffusionStep+1}^{\fSMCParticleIndex}$
                    \State $Z_{\fDiffusionStep+1}^{\fSMCParticleIndex}
                    \gets
                    \fPartitionFuncClamp{}
                    {\fDiffState{\fDiffusionStep+1}^{\fSMCParticleIndex}}
                    {\fDiffState{\fDiffusionStep+1}^{\fSMCParticleIndex}}$
                    \Comment{successor partition sum and cache \ffbs messages}
                \Else
                    \State $Z_{\fDiffusionStepMax}^{\fSMCParticleIndex}
                    \gets
                    \fIndicatorFunc{
                    \fDiffState{\fDiffusionStepMax}^{\fSMCParticleIndex}
                    \in\fLangSet}$
                \EndIf
            \EndParallelFor
            \Statex \AtsmcStage{lightpurple}{RE-FREEZING CORRECTION}{6}
            \ParallelFor{$\fSMCParticleIndex=1,\ldots,\fSMCParticleIndexMax$}
                \State
                $\fFeynmanKacPotentialSymbol_{\fDiffusionStep+1}^{\fSMCParticleIndex}
                \gets
                {Z_{\fDiffusionStep+1}^{\fSMCParticleIndex}}/
                {\widehat Z_{\fDiffusionStep+1}^{\fSMCParticleIndex}}$
                \State $\fSMCParticleWeightUNorm{\fDiffusionStep+1}^{\fSMCParticleIndex}
                \gets
                \fSMCParticleWeight{\fDiffusionStep}^{\fSMCParticleIndex}
                \fFeynmanKacPotentialSymbol_{\fDiffusionStep+1}^{\fSMCParticleIndex}$
            \EndParallelFor
            \State $\fSMCParticleWeight{\fDiffusionStep+1}^{\fSMCParticleIndex}
            \gets
            \fSMCParticleWeightUNorm{\fDiffusionStep+1}^{\fSMCParticleIndex}/
            {\sum_{\fSMCParticleIndex'=1}^{\fSMCParticleIndexMax}
            \fSMCParticleWeightUNorm{\fDiffusionStep+1}^{\fSMCParticleIndex'}}$
            for all $\fSMCParticleIndex\in[\fSMCParticleIndexMax]$
            \Statex \AtsmcStage{lightpurple}{\hspace*{-5mm}ADAPTIVE RESAMPLING}{6}
            \State $\mathrm{ESS}\gets
            \bigl(\sum_{\fSMCParticleIndex'=1}^{\fSMCParticleIndexMax}
            (\fSMCParticleWeight{\fDiffusionStep+1}^{\fSMCParticleIndex'})^2
            \bigr)^{-1}$
            \If{$\fDiffusionStep<\fDiffusionStepMax\!-\!1
            \;\;\textbf{and}\;\; \mathrm{ESS}
            <\mathrm{ESS}_{\mathrm{min}}\fSMCParticleIndexMax$}
                \State $\fDiffState{\fDiffusionStep+1}^{1:\fSMCParticleIndexMax}
                \gets\Call{Resample}{
                \fDiffState{\fDiffusionStep+1}^{1:\fSMCParticleIndexMax},
                \fSMCParticleWeight{\fDiffusionStep+1}^{1:\fSMCParticleIndexMax}}$
                \Comment{resample cache jointly}
                \State $\fSMCParticleWeight{\fDiffusionStep+1}^{\fSMCParticleIndex}
                \gets1/\fSMCParticleIndexMax$
                for all $\fSMCParticleIndex\in[\fSMCParticleIndexMax]$
            \EndIf
        \EndFor
        \State \Return
        $(\fDiffState{\fDiffusionStepMax}^{1:\fSMCParticleIndexMax},
                \fSMCParticleWeight{\fDiffusionStepMax}^{1:\fSMCParticleIndexMax})$
    \end{algorithmic}
    \alglanguage{pseudocode}
    \end{algorithm}

Algorithm~\ref{alg:twister-full} gives the complete \tool algorithm. It requires a \dfa $\fAutomtn$ (lifted to the model's vocabulary $\fVocab$) that recognizes the regular language $\fLangSet$; an initial state $\fDiffState{0}$, \ie a token block containing the prompt and masked tokens to be unmasked; the maximum number of denoising steps $\fDiffusionStepMax$; the number of \smc particles $\fSMCParticleIndexMax$; and the minimum effective sample size $\mathrm{ESS}_{\mathrm{min}}$, \ie the threshold for triggering resampling of the \smc particles. Resampling is used in \smc to mitigate \emph{weight degeneracy}: when weights of most \smc particles become too small, rendering them useless and thus wasting compute.

\tool starts by querying the denoiser at the initial state $\fDiffState{0}$, computing its partition sum \smash{\footnotesize$Z_0$} and caching its \ffbs messages. Following this, it initializes all $\fSMCParticleIndexMax$ \smc particles to the initial state $\fDiffState{0}$ with weight $1/\fSMCParticleIndexMax$. The cached \ffbs messages are used to sample from the \fkc proposal \smash{\footnotesize{$\fFeynmanKacKernelSymbol_{\fDiffusionStep+1}$}}, \ie the step-exact proposal of Proposition~\ref{prop:ffbs-frozen-doob}, and compute the predecessor clamped-partition sum \smash{\footnotesize$\widehat Z_{\fDiffusionStep+1}^{\fSMCParticleIndex}$}. For each propagated nonterminal particle, the denoiser is queried at the successor state to compute the successor partition sum \smash{\footnotesize$Z_{\fDiffusionStep+1}^{\fSMCParticleIndex}$} and cache the \ffbs messages for the next step.
The ratio
\smash{\footnotesize$\fFeynmanKacPotentialSymbol_{\fDiffusionStep+1}^{\fSMCParticleIndex}
=Z_{\fDiffusionStep+1}^{\fSMCParticleIndex}/
\widehat Z_{\fDiffusionStep+1}^{\fSMCParticleIndex}$} gives the exact \fkc potential required in Section~\ref{sec:method-increment} for fixing the tilt from
Proposition~\ref{prop:re-freezing-tilt}. Each incoming particle weight is multiplied by this potential, and the resulting weights are normalized. Particles are jointly resampled with their caches and assigned uniform weight when resampling is triggered by the effective sample size falling below $\mathrm{ESS}_{\mathrm{min}}\fSMCParticleIndexMax$.

\MyPara{Complexity}
One weighted-automaton pass costs
$\fBigO{\fPositionMax\fSize{\fStates}\fSize{\fVocab}}$, and so, the total automaton cost for $\fSMCParticleIndexMax$ \smc particles and $\fDiffusionStepMax$ denoising steps is $\fBigO{\fSMCParticleIndexMax\fDiffusionStepMax
\fPositionMax\fSize{\fStates}\fSize{\fVocab}}$, in addition to the denoiser evaluations. \tool is parallelizable over the step-exact \ffbs propagate and re-freezing correction phases, since the particles are conditionally independent until weight normalization and resampling. 
\section{Constraint Satisfaction Experiment}
\label{sec:appendix:eval}

\begin{figure}[t]
    \centering
    \includegraphics[width=0.55\textwidth]{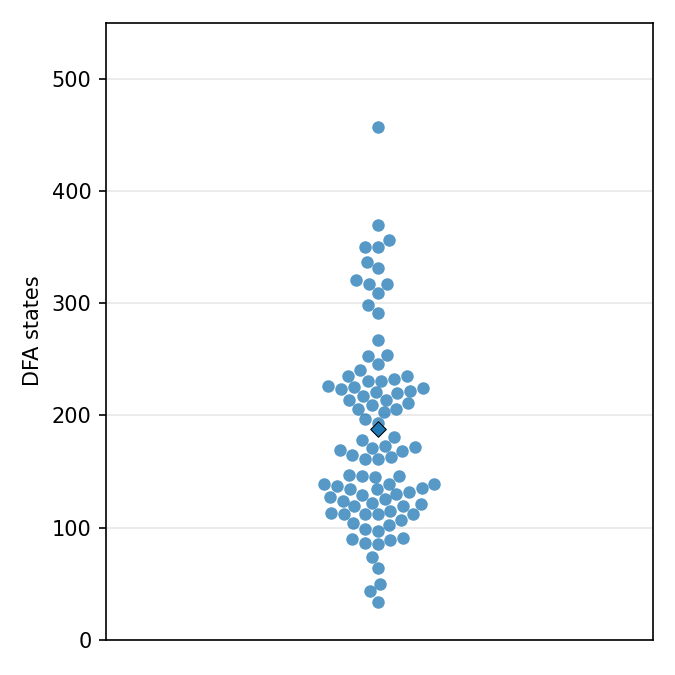}
    \caption{Swarm plot of token-prefix DFA state counts on \jsonModeEval.
    Each point is one datapoint; the diamond marks the mean.}
    \label{fig:dfa-state-sizes}
\end{figure}

\begin{figure}[t]
\centering
\begin{tcolorbox}[
  colback=gray!4,
  colframe=black!55,
  boxrule=0.4pt,
  arc=2pt,
  left=3pt,
  right=3pt,
  top=2pt,
  bottom=2pt,
]
\begin{lstlisting}[basicstyle=\ttfamily\scriptsize,breaklines=true,breakatwhitespace=true,columns=fullflexible,keepspaces=true,aboveskip=0pt,belowskip=0pt]
[system]
You are a helpful assistant that answers in JSON. Here's the json schema you must adhere to:
<schema>
{'title': 'WirelessAccessPoint', 'type': 'object', 'properties': {'ssid': {'title': 'SSID', 'type': 'string'}, 'securityProtocol': {'title': 'SecurityProtocol', 'type': 'string'}, 'bandwidth': {'title': 'Bandwidth', 'type': 'string'}}, 'required': ['ssid', 'securityProtocol', 'bandwidth']}
</schema>

[user]
I'm currently configuring a wireless access point for our office network and I need to generate a JSON object that accurately represents its settings. The access point's SSID should be 'OfficeNetSecure', it uses WPA2-Enterprise as its security protocol, and it's capable of a bandwidth of up to 1300 Mbps on the 5 GHz band. This JSON object will be used to document our network configurations and to automate the setup process for additional access points in the future. Please provide a JSON object that includes these details.
Only output the JSON object, no other text or comments.
\end{lstlisting}
\end{tcolorbox}
\caption{Example \jsonModeEval prompt messages (datapoint \texttt{json-0}).
The final user-turn line is our appended instruction to emit only the JSON object.}
\label{fig:json-mode-prompt}
\end{figure}

As a sanity check that \tool still produces only constraint-satisfying outputs, we evaluate it along with other constrained decoding methods on the \jsonModeEval dataset released by~\cite{nousresearch2024jsonmodeeval}.
\jsonModeEval consists of 100 zero-shot \datapoints, each specifying a given schema and requiring a JSON output that conforms to that schema.
Of the 100 \datapoints in \jsonModeEval, we exclude six whose schemas cannot be compiled into a \regex by \outlines (five use a named-root wrapper rather than a valid JSON Schema root, and one is malformed), and evaluate on the remaining 94.
For each remaining \datapoint, we compile the JSON schema into a \regex with the \outlines library from \citet{willardOutlines23} and transform that \regex into a token-prefix automaton over the model vocabulary.
Figure~\ref{fig:dfa-state-sizes} shows the distribution of token-prefix \dfa sizes created for the 94 evaluated schemas.
The distribution is concentrated in the low hundreds of states, with most values falling between roughly $100$ and $250$.
The mean number of states is \dfaStatesMean, with a minimum of \dfaStatesMin states and a maximum of \dfaStatesMax states.

We build the dataset prompt for each \datapoint, append an instruction to emit only the JSON object, and then decode under each of the constrained strategies described below. 
Figure~\ref{fig:json-mode-prompt} shows an example of the prompt message for a \jsonModeEval \datapoint given to the model.

We compare four constrained decoding strategies.
\OurDingo, at each denoising step, selects a constraint-satisfying completion that maximizes the automaton-weighted mean-field posterior.
\OurDingo picks the mode by using dynamic programming over automaton states to select a locally optimal constraint-satisfying block under the mean-field factorization.
Our \OurDingo implementation follows the original dynamic-programming procedure from~\citet{sureshdingo25}. The original code is not publicly available.
\OurDangErmon instead samples exactly from the automaton-constrained mean-field posterior at each step using \ffbs~\citep{carterGibbs1994,dang2026constrained}, which guarantees local validity but is still subject to the trajectory bias characterized in~\S\ref{sec:problem-trajectory-bias}. 
Our \OurDangErmon implementation follows the setup and procedure described in~\citet{dang2026constrained}. As with \OurDingo, the original code is not publicly available.
Finally, $\gnardSmcOne$ and $\gnardSmcFour$ apply \tool with $K\!=\!1$ and $K\!=\!4$ particles, using the step-exact constrained posterior and correcting trajectory bias with the Feynman--Kac twists described in~\S\ref{sec:method}.

We evaluate all decoding methods on six open \mdlms from the \dreamText~\citep{yeDream25}, \dreamCoder~\citep{xieDreamCoder25}, and \lladaText~\citep{niellada25} families: \dream, \dreamInstruct, \dreamCoderText, \dreamCoderInstruct, \llada, and \lladaInstruct.
We chose these models because they are among the strongest publicly available \mdlms.
We set generation length and denoising steps to $1.5\times$ the gold answer length per \datapoint.

\begin{\UseMacro{table-json-mode-table}}[t]%
\centering%
\caption{\TableJsonModeDatasetCaption}%
\TableFontSize%
\setlength{\tabcolsep}{6pt}%
\renewcommand{\arraystretch}{1.15}%
\resizebox{\textwidth}{!}{\begin{tabular}{llccc}%
\toprule%
\UseMacro{table-json-mode-Model}
 & \UseMacro{table-json-mode-Method}
 & \UseMacro{table-json-mode-Parse}
 & \UseMacro{table-json-mode-Schema}
 & \UseMacro{table-json-mode-Time} \\%
\midrule%
\multirow{4}{*}{\dream}
 & \OurDingo
 & \UseMacro{table-json-mode-Dream-v0-Base-7B-dingo-parse}
 & \UseMacro{table-json-mode-Dream-v0-Base-7B-dingo-schema}
 & \UseMacro{table-json-mode-Dream-v0-Base-7B-dingo-time-mean-Dream-v0-Base-7B-dingo-time-std} \\%

 & \OurDangErmon
 & \UseMacro{table-json-mode-Dream-v0-Base-7B-cedar-marginal-parse}
 & \UseMacro{table-json-mode-Dream-v0-Base-7B-cedar-marginal-schema}
 & \UseMacro{table-json-mode-Dream-v0-Base-7B-cedar-marginal-time-mean-Dream-v0-Base-7B-cedar-marginal-time-std} \\%

 & $\gnardSmcOne$
 & \UseMacro{table-json-mode-Dream-v0-Base-7B-cedar-smc1-parse}
 & \UseMacro{table-json-mode-Dream-v0-Base-7B-cedar-smc1-schema}
 & \UseMacro{table-json-mode-Dream-v0-Base-7B-cedar-smc1-time-mean-Dream-v0-Base-7B-cedar-smc1-time-std} \\%

 & $\gnardSmcFour$
 & \UseMacro{table-json-mode-Dream-v0-Base-7B-cedar-smc4-parse}
 & \UseMacro{table-json-mode-Dream-v0-Base-7B-cedar-smc4-schema}
 & \UseMacro{table-json-mode-Dream-v0-Base-7B-cedar-smc4-time-mean-Dream-v0-Base-7B-cedar-smc4-time-std} \\%
\midrule%
\multirow{4}{*}{\dreamInst}
 & \OurDingo
 & \UseMacro{table-json-mode-Dream-v0-Instruct-7B-dingo-parse}
 & \UseMacro{table-json-mode-Dream-v0-Instruct-7B-dingo-schema}
 & \UseMacro{table-json-mode-Dream-v0-Instruct-7B-dingo-time-mean-Dream-v0-Instruct-7B-dingo-time-std} \\%

 & \OurDangErmon
 & \UseMacro{table-json-mode-Dream-v0-Instruct-7B-cedar-marginal-parse}
 & \UseMacro{table-json-mode-Dream-v0-Instruct-7B-cedar-marginal-schema}
 & \UseMacro{table-json-mode-Dream-v0-Instruct-7B-cedar-marginal-time-mean-Dream-v0-Instruct-7B-cedar-marginal-time-std} \\%

 & $\gnardSmcOne$
 & \UseMacro{table-json-mode-Dream-v0-Instruct-7B-cedar-smc1-parse}
 & \UseMacro{table-json-mode-Dream-v0-Instruct-7B-cedar-smc1-schema}
 & \UseMacro{table-json-mode-Dream-v0-Instruct-7B-cedar-smc1-time-mean-Dream-v0-Instruct-7B-cedar-smc1-time-std} \\%

 & $\gnardSmcFour$
 & \UseMacro{table-json-mode-Dream-v0-Instruct-7B-cedar-smc4-parse}
 & \UseMacro{table-json-mode-Dream-v0-Instruct-7B-cedar-smc4-schema}
 & \UseMacro{table-json-mode-Dream-v0-Instruct-7B-cedar-smc4-time-mean-Dream-v0-Instruct-7B-cedar-smc4-time-std} \\%
\midrule%
\multirow{4}{*}{\dreamCoderText}
 & \OurDingo
 & \UseMacro{table-json-mode-Dream-Coder-v0-Base-7B-dingo-parse}
 & \UseMacro{table-json-mode-Dream-Coder-v0-Base-7B-dingo-schema}
 & \UseMacro{table-json-mode-Dream-Coder-v0-Base-7B-dingo-time-mean-Dream-Coder-v0-Base-7B-dingo-time-std} \\%

 & \OurDangErmon
 & \UseMacro{table-json-mode-Dream-Coder-v0-Base-7B-cedar-marginal-parse}
 & \UseMacro{table-json-mode-Dream-Coder-v0-Base-7B-cedar-marginal-schema}
 & \UseMacro{table-json-mode-Dream-Coder-v0-Base-7B-cedar-marginal-time-mean-Dream-Coder-v0-Base-7B-cedar-marginal-time-std} \\%

 & $\gnardSmcOne$
 & \UseMacro{table-json-mode-Dream-Coder-v0-Base-7B-cedar-smc1-parse}
 & \UseMacro{table-json-mode-Dream-Coder-v0-Base-7B-cedar-smc1-schema}
 & \UseMacro{table-json-mode-Dream-Coder-v0-Base-7B-cedar-smc1-time-mean-Dream-Coder-v0-Base-7B-cedar-smc1-time-std} \\%

 & $\gnardSmcFour$
 & \UseMacro{table-json-mode-Dream-Coder-v0-Base-7B-cedar-smc4-parse}
 & \UseMacro{table-json-mode-Dream-Coder-v0-Base-7B-cedar-smc4-schema}
 & \UseMacro{table-json-mode-Dream-Coder-v0-Base-7B-cedar-smc4-time-mean-Dream-Coder-v0-Base-7B-cedar-smc4-time-std} \\%
\midrule%
\multirow{4}{*}{\dreamCoderInst}
 & \OurDingo
 & \UseMacro{table-json-mode-Dream-Coder-v0-Instruct-7B-dingo-parse}
 & \UseMacro{table-json-mode-Dream-Coder-v0-Instruct-7B-dingo-schema}
 & \UseMacro{table-json-mode-Dream-Coder-v0-Instruct-7B-dingo-time-mean-Dream-Coder-v0-Instruct-7B-dingo-time-std} \\%

 & \OurDangErmon
 & \UseMacro{table-json-mode-Dream-Coder-v0-Instruct-7B-cedar-marginal-parse}
 & \UseMacro{table-json-mode-Dream-Coder-v0-Instruct-7B-cedar-marginal-schema}
 & \UseMacro{table-json-mode-Dream-Coder-v0-Instruct-7B-cedar-marginal-time-mean-Dream-Coder-v0-Instruct-7B-cedar-marginal-time-std} \\%

 & $\gnardSmcOne$
 & \UseMacro{table-json-mode-Dream-Coder-v0-Instruct-7B-cedar-smc1-parse}
 & \UseMacro{table-json-mode-Dream-Coder-v0-Instruct-7B-cedar-smc1-schema}
 & \UseMacro{table-json-mode-Dream-Coder-v0-Instruct-7B-cedar-smc1-time-mean-Dream-Coder-v0-Instruct-7B-cedar-smc1-time-std} \\%

 & $\gnardSmcFour$
 & \UseMacro{table-json-mode-Dream-Coder-v0-Instruct-7B-cedar-smc4-parse}
 & \UseMacro{table-json-mode-Dream-Coder-v0-Instruct-7B-cedar-smc4-schema}
 & \UseMacro{table-json-mode-Dream-Coder-v0-Instruct-7B-cedar-smc4-time-mean-Dream-Coder-v0-Instruct-7B-cedar-smc4-time-std} \\%
\midrule%
\multirow{4}{*}{\llada}
 & \OurDingo
 & \UseMacro{table-json-mode-LLaDA-Base-8B-dingo-parse}
 & \UseMacro{table-json-mode-LLaDA-Base-8B-dingo-schema}
 & \UseMacro{table-json-mode-LLaDA-Base-8B-dingo-time-mean-LLaDA-Base-8B-dingo-time-std} \\%

 & \OurDangErmon
 & \UseMacro{table-json-mode-LLaDA-Base-8B-cedar-marginal-parse}
 & \UseMacro{table-json-mode-LLaDA-Base-8B-cedar-marginal-schema}
 & \UseMacro{table-json-mode-LLaDA-Base-8B-cedar-marginal-time-mean-LLaDA-Base-8B-cedar-marginal-time-std} \\%

 & $\gnardSmcOne$
 & \UseMacro{table-json-mode-LLaDA-Base-8B-cedar-smc1-parse}
 & \UseMacro{table-json-mode-LLaDA-Base-8B-cedar-smc1-schema}
 & \UseMacro{table-json-mode-LLaDA-Base-8B-cedar-smc1-time-mean-LLaDA-Base-8B-cedar-smc1-time-std} \\%

 & $\gnardSmcFour$
 & \UseMacro{table-json-mode-LLaDA-Base-8B-cedar-smc4-parse}
 & \UseMacro{table-json-mode-LLaDA-Base-8B-cedar-smc4-schema}
 & \UseMacro{table-json-mode-LLaDA-Base-8B-cedar-smc4-time-mean-LLaDA-Base-8B-cedar-smc4-time-std} \\%
\midrule%
\multirow{4}{*}{\lladaInst}
 & \OurDingo
 & \UseMacro{table-json-mode-LLaDA-Instruct-8B-dingo-parse}
 & \UseMacro{table-json-mode-LLaDA-Instruct-8B-dingo-schema}
 & \UseMacro{table-json-mode-LLaDA-Instruct-8B-dingo-time-mean-LLaDA-Instruct-8B-dingo-time-std} \\%

 & \OurDangErmon
 & \UseMacro{table-json-mode-LLaDA-Instruct-8B-cedar-marginal-parse}
 & \UseMacro{table-json-mode-LLaDA-Instruct-8B-cedar-marginal-schema}
 & \UseMacro{table-json-mode-LLaDA-Instruct-8B-cedar-marginal-time-mean-LLaDA-Instruct-8B-cedar-marginal-time-std} \\%

 & $\gnardSmcOne$
 & \UseMacro{table-json-mode-LLaDA-Instruct-8B-cedar-smc1-parse}
 & \UseMacro{table-json-mode-LLaDA-Instruct-8B-cedar-smc1-schema}
 & \UseMacro{table-json-mode-LLaDA-Instruct-8B-cedar-smc1-time-mean-LLaDA-Instruct-8B-cedar-smc1-time-std} \\%

 & $\gnardSmcFour$
 & \UseMacro{table-json-mode-LLaDA-Instruct-8B-cedar-smc4-parse}
 & \UseMacro{table-json-mode-LLaDA-Instruct-8B-cedar-smc4-schema}
 & \UseMacro{table-json-mode-LLaDA-Instruct-8B-cedar-smc4-time-mean-LLaDA-Instruct-8B-cedar-smc4-time-std} \\%
\bottomrule%
\end{tabular}}%
\end{\UseMacro{table-json-mode-table}}%

Table~\ref{tab:json-mode-dataset} shows the results of our evaluation of the different decoding methods on the \jsonModeEval dataset.
We report the parse validity, schema validity, and average time in seconds of the generated outputs.
Parse validity ensures that the generated output is valid JSON, while schema validity ensures that the generated output conforms to the expected schema set by the specific dataset sample.
Importantly, all decoding methods, including \tool, consistently produce outputs that are valid JSON and conform to the expected schema.
Although we report average generation times, our implementation is not optimized for performance.
There are techniques to speed up the decoding time such as those described in~\citet{dang2026constrained}.
This experiment acts as a confirmation to show that the trajectory-level correction we achieve with \tool does not compromise the constraint satisfaction already guaranteed by \ffbsKernel decoding.

\end{document}